\documentclass{article}
\pdfoutput=1
\usepackage{arxiv}
\usepackage{cite}

\usepackage{amsthm}
\newtheorem{theorem}{Theorem}[section]
\newtheorem{lemma}{Lemma}[section]

\newtheorem{remark}{Remark}[section]

\newcommand{\insertkeywords}{\keywords{
Full Approximation Scheme \and FAS \and Multigrid \and Multilevel Methods \and DNN Training \and Neural Networks}}

\newcommand\blfootnote[1]{%
  \begingroup
  \renewcommand\thefootnote{}\footnote{#1}%
  \addtocounter{footnote}{-1}%
  \endgroup
}

\newcommand{\inserttitle}{\title{Multilevel-in-Layer Training for Deep Neural Network Regression}}

\newcommand{\insertabstract}[1]{\begin{abstract}#1\end{abstract}}

\newcommand{\insertauthors}{\author{Colin Ponce \\
Center for Applied Scientific Computing \\
Lawrence Livermore National Laboratory \\
\tt{ponce11@llnl.gov}
\And
Ruipeng Li \\
Center for Applied Scientific Computing \\
Lawrence Livermore National Laboratory \\
\tt{li50@llnl.gov}
\And
Christina Mao \\
Global Security Computing Applications Division \\
Lawrence Livermore National Laboratory \\
\tt{mao6@llnl.gov}
\And
Panayot Vassilevski \\
Fariborz Maseeh Department of Mathematics + Statistics \\
Portland State University \\
\tt{panayot@pdx.edu}
}\blfootnote{This work was performed under the auspices of the U.S. Department of Energy by Lawrence Livermore National Laboratory under Contract DE-AC52-07NA27344 (LLNL-JRNL-827231) and was supported by the LLNL-LDRD program under Project No. 19-ERD-019.}}

\newcommand{\insertbibstyle}{\bibliographystyle{plain}}

\input epsf
\usepackage{epsfig}
\usepackage{xcolor}
\usepackage{fullpage}
\usepackage{amsmath,amssymb,amsfonts}
\usepackage{graphicx}
\usepackage{enumerate}
\usepackage{bbm}
\usepackage{url,hyperref}
\usepackage{float}
\usepackage{multirow}
\usepackage{subfig}
\usepackage{placeins}   
\usepackage{enumitem}
\usepackage{subfig}
\usepackage{algorithm, algpseudocode}
\usepackage[american]{babel}

\renewcommand{\div}{\operatorname{div}}

\newcommand{\bb}{\ensuremath{\mathbf{b}}}

\newcommand{\bs}{\ensuremath{\mathbf{s}}}

\newcommand{\bq}{\ensuremath{\mathbf{q}}}
\newcommand{\bx}{\ensuremath{\mathbf{x}}}
\newcommand{\by}{\ensuremath{\mathbf{y}}}
\newcommand{\bz}{\ensuremath{\mathbf{z}}}
\newcommand{\bn}{\ensuremath{\mathbf{n}}}
\newcommand{\bbf}{\ensuremath{\mathbf{f}}}

\newcommand{\bV}{\ensuremath{\mathbf{V}}}

\newcommand{\bH}{\ensuremath{\mathbf{H}}}

\newcommand{\bP}{\ensuremath{\mathbf{P}}}
\newcommand{\bQ}{\ensuremath{\mathbf{Q}}}

\newcommand{\bR}{\ensuremath{\mathbf{R}}}
\newcommand{\bc}{\ensuremath{\mathbf{c}}}
\newcommand{\bm}{\ensuremath{\mathbf{m}}}

\newcommand{\btau}{\ensuremath{\mathbf{\tau}}}

\def\balpha{{\boldsymbol \alpha}}

\newcommand{\T}{{\mathcal T}}

\def\bPi{{\boldsymbol \Pi}}

\def\bP{{\mathbf P}}

\begin{document}

\inserttitle

\insertauthors
\maketitle

\insertabstract{A common challenge in regression is that for many problems, the degrees of freedom required for a high-quality solution also allows for overfitting. Regularization is a class of strategies that seek to restrict the range of possible solutions so as to discourage overfitting while still enabling good solutions, and different regularization strategies impose different types of restrictions. In this paper, we present a multilevel regularization strategy that constructs and trains a hierarchy of neural networks, each of which has layers that are wider versions of the previous network's layers. We draw intuition and techniques from the field of Algebraic Multigrid (AMG), traditionally used for solving linear and nonlinear systems of equations, and specifically adapt the Full Approximation Scheme (FAS) for nonlinear systems of equations to the problem of deep learning. Training through V-cycles then encourage the neural networks to build a hierarchical understanding of the problem. We refer to this approach as \emph{multilevel-in-width} to distinguish from prior multilevel works which hierarchically alter the depth of neural networks. The resulting approach is a highly flexible framework that can be applied to a variety of layer types, which we demonstrate with both fully-connected and convolutional layers. We experimentally show with PDE regression problems that our multilevel training approach is an effective regularizer, improving the generalize performance of the neural networks studied.}

\insertkeywords

\section{Introduction}\label{introduction}

A feedforward deep neural network (DNN) consists of a set of layers, each of which applies an affine mapping to its input vectors followed by a nonlinear \emph{activation function} acting componentwise on the transferred vectors. That affine mapping is controlled by a set of tunable parameters, the selection of which determines the overall action of the DNN. The task of training a deep neural network (DNN) is that of selecting these network parameters in order to enable the DNN to approximately capture a desired pattern. This is done by solving, or approximately solving, an optimization problem with an objective function that maps a choice of DNN parameters to a quantitative measure of the DNN's performance on the problem of interest; this objective function is commonly called a \emph{loss function} \cite{Goodfellow-et-al-2016}.

Training is typically accomplished through stochastic gradient descent (SGD) or its variants, which are cost-efficient and well-supported by existing software \cite{pytorch_neurips2019_9015, tensorflow2015-whitepaper}. A common challenge in the design of neural networks and their training is that, often, providing sufficient degrees of freedom in the network to obtain a high-quality solution also allows for overfitting, which harms generalization accuracy. \emph{Regularization} is a class of strategies that seeks to restrict the freedom of the training method, therefore encouraging simpler solutions. Ideally, this regularization discourages overfitting while still enabling high-quality solutions. A variety of regularization strategies exist, each of which encourages a different kind of simplicity, such as L2 regularization, L1 regularization, or dropout regulariztaion \cite{Goodfellow-et-al-2016}.

In this paper, we present a novel \emph{multilevel} regularization strategy that constructs and trains a hierarchy of neural networks. In this approach, all networks have the same number and type of layers, but each layer is a wider version of next network's layers, so networks get "thinner" as we move through the hierarchy. We base this approach on the techniques of Algebraic Multigrid (AMG), a class of methods traditionally used for solving linear and nonlinear systems of equations, and we leverage teh Full Approximation Scheme (FAS), a multilevel framework traditionally used for nonlinear systems of equations. We use the classic V-cycle in training, and develop a method to pass problem and parameter information, including momentum information, between the levels of the hierarchy during training. The resulting training method encourages the neural networks to learn the problem in a hierarchical manner. This approach is highly flexible, allowing each layer to coarsen (or not) independently of each other layer, and can be applied to a variety of layer types. We specifically demonstrate this for fully-connected and convolutional layers. Our approach is not a replacement for traditional, one-level learning methods such as SGD; instead, it uses that traditional optimizer as an inner iterator within the multilevel framework. We demonstrate the generalization benefits of our approach with a set of PDE-based regression problems. 

We refer to our approach as "multilevel-in-width" to distinguish it from prior multilevel works that change the number of layers in a neural network, which we think of as "multilevel-in-depth." Other authors, taking very different approaches from that presented here, have also observed regularization benefits from their multilevel methods \cite{planta2021resnets, cyr2019multilevelinitialization}. The present paper contributes to a small but growing body of evidence that suggests that multilevel neural network training methods have the potential to act as a beneficial regularizer over a significant range of problems, network architectures, and multilevel approaches, though this requires further investigation to fully understand.

Our open-source software, which we call MTNN (short for \emph{Multilevel Training of Neural Networks}, pronounced ``mountain'') is built on PyTorch and is free to download at \url{https://github.com/LLNL/MTNN}. It is designed for modularity and extensibility to facilitate the study of new problems and algorithmic choices.

In Section~\ref{background}, we discuss prior work in AMG, current methods for training neural networks, and prior work on building notions of hierarchy into DNNs. We then give a summary of the FAS algorithm from an optimization point of view in Section~\ref{FAS for optimization problem}. Next, in Section~\ref{FAS applied to DNN training}, we formulate a two-level FAS in the setting of DNN training. 
Our specific choices of the operators (restriction and prolongation) are given in Section~\ref{FAS restriction and prolongation}. 
Section~\ref{DNN FAS algorithm} presents our two-level full-approximation scheme, which can be recursively extended to more levels. Section~\ref{convolutional extension} presents the extension of our method to convolutional neural networks. Finally, in Section~\ref{experiments} we demonstrate our techniques on a set of PDE regression problems, and in Section~\ref{conclusion} we conclude.

\section{Background and Related Work}\label{background}

\subsection{Multilevel Methods}
Multilevel methods are a class of iterative methods long used in scientific computing as a scalable solver for linear and non-linear systems of equations $g(\bx) = \bbf$. Multilevel methods is a more general class of algorithms emerging out of \emph{multigrid methods}, where this method was first applied to partial differential problems on grid. Traditional multigrid algorithms construct a hierarchy of discretizations of a geometric problem domain in which each discretization is a coarsened variant of the previous. Algebraic multigrid (AMG) also constructs a hierarchy of problems, but does in a fully-algebraic manner, independent of any underlying geometry.

When applied to (deterministic) systems of equations, multilevel methods accelerate convergence towards a solution through \emph{smoothing}, also known as relaxation methods, and \emph{coarse-grid corrections}, down-sampling residual error on the coarse grid back to the fine grid. Passing solution information between fine and coarse grids is known as \emph{restriction} and \emph{prolongation}.

A single multilevel iteration is known as a \emph{cycle}. Perhaps the most common of these is the V-cycle, the linear two-level version of which proceeds as follows:
\begin{enumerate}
    \item Smooth the solution iterate on the fine level problem $g$.
    \item Restrict problem residual to the coarse level problem $g_{c}$ with some restriction operator $R$.
    \item Solve the problem on the coarse level to obtain a coarse-grid correction.
    \item Prolong the coarse-grid correction to the fine level and add it to the solution iterate with some prolongation operator $P$.
    \item Smooth the solution iterate on the fine level problem $g$.
\end{enumerate}
This V-cycle is a framework in which one must make algorithmic choices on how to accomplish individual steps such as smoothing, restriction, and prolongation. 
For further reference, a more detailed introduction to multilevel and multigrid methods can be provided by Briggs et al. \cite{briggs-multigrid}, Hackbusch \cite{hackbusch2013multi} and Wesseling \cite{wesseling1995introduction}.

A classic problem for which linear multilevel methods are extremely effective is the Poisson problem, $-\nabla~\cdot~(\kappa~\nabla~u)~=~f$, defined over some problem domain, which models diffusion processes. As a result, data from Poisson problems are often amenable to the type of hierarchical representations discussed in the Introduction.

\subsubsection{Full Approximation Scheme} When multilevel methods are applied to nonlinear systems of equations, the V-cycle is similar, but passing the residual to the coarse level may no longer be appropriate. A generalization appropriate for nonlinear problems is the \emph{Full Approximation Scheme} (FAS) \cite{brandt1977fas}. In FAS, one still has a fine problem, $g$, a coarse problem, $g_c$, and restriction and prolongation operators, $R$ and $P$.

However, instead of restricting the residual as in Step 2 above, one instead restricts the solution iterate, and alters the right-hand side of the coarse problem to capture residual information. The FAS coarse-grid problem, then, is
\begin{equation}\label{FAS coarse problem}
{\widehat g_c}(\bx_c) = R \bb + (g_c(R \bx) - R g(\bx)).
\end{equation}
The above modification to the right-hand side is sometimes called a \emph{$\tau$-correction}; one can check that in the linear case it reduces back to the coarse-grid residual equation.

An extension of FAS to discretized optimization problems, ``MG/OPT'', was introduced by Nash, who also demonstrated its convergence properties \cite{Nash2000discretized}. The MG/OPT algorithm extends FAS by performing a line search on applying the coarse-grid correction (Step 4 in the above V-cycle).

\subsection{Deep Neural Networks (DNN)}
A neural network (NN) consists of $n_L$ layers: an input layer, $n_L-2$ hidden layers, and an output layer. An input vector is denoted by $\by^{in}$. The hidden layers consists of weight matrices $W_k$ and bias vectors $\bb_k$, $k=1,...,n_L$, and an activation function $\varphi$ such as the well-known ReLU function.  The output of the neural network calculation is denoted by $\by^{out}$. There is no clear definition for the number of hidden layers required to consider a neural network a ``deep'' neural network, though often any network with more than 2 hidden layers is considered a DNN.

Forward propagation, also known as a \emph{forward pass}, is the procedure whereby the input (feature) vector is passed through the neural network layer by layer, applying an activation function after each layer to produce intermediate variables and finally an output vector. We discuss the notation and operation of a forward pass in greater detail in Section~\ref{FAS applied to DNN training}.

\subsubsection{Training Methods for Neural Networks}\label{sec:training_methods_for_nn}

An important difference between neural network training and traditional optimization is that the true objective function of interest, performance over the entire population of possible examples, is inaccessible; the performance over a training sample is merely an approximation to this true function of interest. This motivates a number of important design choices in neural network training methods. Training is typically done via a stochastic first-order method, such as stochastic gradient descent, AdaGrad \cite{duchi-adagrad}, RMSProp, Adam \cite{kingma-adam}, or other variations \cite{ma-quasi-hyperbolic, lucas-momentum}; this both accelerates computation and reduces the chances of getting stuck in a local minimum that is an artifact is the particular training set. Momentum maintains an exponentially decaying sum of past gradients, which helps to accelerate training and jump over local minima by allowing the "big ideas" to compound \cite{sutskever-momentum}. Finally, training all the way to convergence on the training set would likely result in severe overfitting, even with regularization, and so early stopping is almost universal \cite{Goodfellow-et-al-2016}. 

These methods are all fundamentally variations on gradient descent. Second-order methods such as Newton-Raphson have been standard in optimization for some time \cite{avriel-programming}, and in fact is often much more effective than gradient descent-style methods. However, second order methods require the computation or approximation of the Hessian matrix of the objective function, requiring the computation of $O(n^2)$ Hessian elements for $n$ parameters. In the case of DNNs, even moderately sized networks contain millions of parameters, and so a Hessian matrix is very slow to compute (and sometimes does not fit in memory); consequently, second-order methods have not become mainstream for DNNs. Approximate second-order methods for DNNs with linear storage requirements, however, continues to be an active area of research and are beginning to become more competitive \cite{tan-second-order}.

\subsubsection{Hierarchy and Multilevel Methods in Neural Network Architectures}

Although SGD and its variations do not naturally capture hierarchy, notions of hierarchy are quite common in neural network architectures. Convolution neural networks (CNNs) for image processing are often described as building up concepts in a hierarchical manner, capturing simple ideas such as edges and shading in the first layer, and capturing more complex image elements in later layers. Beyond CNNs, the autoencoder is one of the most popular architectural designs today \cite{Goodfellow-et-al-2016}. There are many variations and applications of autoencoders but they all fundamentally work by hierarchically compressing input data into a low-dimension representation before hierarchically decompressing it back to its original size.

Another network architecture that deserves special mention is U-net \cite{ronneberger-unet}. This network, through its layers, progressively makes an images smaller and smaller (less height and width) but deeper and deeper (more ``color'' channels), before then passing the data through layers that make it smaller and shallower again, eventually producing a processed image of the original size. The architecture design is highly analogous to the V-Cycle that is ubiquitous in multilevel methods \cite{briggs-multigrid}, and works because it processes images in a hierarchical fashion.

In recent years, a small but growing body of literature has explored multilevel methods in the training itself. In \cite{calandra2020levenbergmarquardt}, the authors develop a multilevel-in-width Levenberg-Marquardt method for training a neural network with one hidden layer to find a discretization-free approximate solution of a PDE. Most other works currently in this space focus on multilevel-in-depth methods for deep residual networks (ResNets), exploiting the dynamical perspective that views a ResNet as taking forward Euler steps on an ODE: G{\"u}nther et. al. \cite{gunther2020layerparallel} uses multilevel parallel-in-time methods to achieve parallelism in the forward and backward propagation; Kirby et. al. \cite{kirby2020GPU} extends this approach to GPUs,  and Cyr et. al. \cite{cyr2019multilevelinitialization} uses this framework to develop a cascadic initialization strategy with regularization benefits. Gaedke-Merzh{\"a}user et. al., \cite{gaedke20201multilevel} and Planta et. al. \cite{planta2021resnets} also study ResNets, but develop variants of MG/OPT to solve the problem in contrast to the parallel-in-time approach.

Both \cite{planta2021resnets} and \cite{cyr2019multilevelinitialization} show that, even though multilevel methods are often motivated by speed, the networks that result from multilevel training methods can have better performance than those trained without them. We observe this in our present work as well.

\section{Application of Full Approximation Scheme (FAS) to an optimization problem}\label{FAS for optimization problem}

In this section, we develop the framework for applying the Full
Approximation Scheme (FAS) to optimization problems and present some results
that will guide the development of FAS-based methods for deep
learning.

Our optimization problem of interest is 
\begin{equation}\label{fine level problem}
g(\bx) \mapsto \min,
\end{equation}
for a given nonlinear functional $g:\;{\mathbb R}^n\mapsto {\mathbb
  R}$. We assume that the gradient of $g$, $\nabla g$ is
computationally available, and we will be using (in the presentation)
the Hessian of $g$, denoted $H_g$, which is a matrix of its second partial
derivatives. Using a Taylor expansion, we have the following
approximate expansion
\begin{equation*}
g(\bx + \bq) \approx g(\bx) + \bq^T \nabla g(\bx) +\frac{1}{2} \bq^T H_g(\bx) \bq.
\end{equation*}

Now, to define a two-level FAS, we need a coarse version of $g$,
$g_c:\; {\mathbb R}^{n_c} \mapsto {\mathbb R}$, where $n_c \ll n$. We
introduce three mappings:
\begin{itemize}
\item A coarse approximation to a fine quantity,
$\bx_c = \bPi \bx$ for a mapping $\bPi:\; {\mathbb R}^n \mapsto {\mathbb R}^{n_c}$.
\item A fine-to-coarse restriction mapping 
$\bR:\; {\mathbb R}^n \mapsto {\mathbb R}^{n_c}$.
\item A coarse-to-fine interpolation mapping 
$\bP:\; {\mathbb R}^{n_c} \mapsto {\mathbb R}^n$.
\end{itemize}

To be precise, given a fine vector $\bx$, the coarse vector 
$\bx_c = \bPi \bx$ is such that $\bP\bx_c = \bP \bPi \bx$ approximates (in some norm) the vector $\bx$. 

The coarse FAS functional ${\widehat g}_c$ is defined in terms of
the coarse functional $g_c$, and in TL FAS is chosen so that
${\widehat g}_c$ approximates $g$. Specifically, let $\bx$ be a
current approximation to the solution of the original problem. Then,
the FAS method is based on the solution of the auxiliary coarse
problem
\begin{equation}\label{eq:tau_corrected_aux}
{\widehat g_c}(\bx_c)\equiv 
g_c(\bx_c) - \bx^T_c\left (\nabla g_c(\bPi \bx) -
\bR \nabla g(\bx) \right )
\mapsto \min .
\end{equation}
Once $\bx_c$ is computed, the next approximation to \eqref{fine level problem} is 
\begin{equation}\label{TL FAS correction}
\bx : = \bx + \bP (\bx_c - \bPi \bx).
\end{equation}

The term $\nabla g_c(\bPi \bx) - \bR \nabla g(\bx)$ is sometimes
refereed to as the \emph{$\tau$-correction}. It replaces $\nabla
g_c(\bPi \bx)$, the gradient of the coarse functional at the coarse
approxmation point, with $\bR \nabla g(\bx)$, the restriction of the
fine gradient. When computed exactly, this makes ${\widehat g_c}(\bx_c)$ match, up to first
order, a restricted version of $g(\bx)$ (though as we will see in Section \ref{sec:tau_construction} we do not compute this exactly).

\subsection{The FAS coarse problem solution}
Let $\bx_c$ solve the coarse problem \eqref{eq:tau_corrected_aux}, so that
$\nabla {\widehat g}(\bx_c) = 0$.
Hence,
\begin{equation*}
\nabla g_c(\bx_c) - \nabla g_c(\bPi \bx) + \bR \nabla g(\bx) = 0.
\end{equation*}
Using Taylor expansion of $\nabla g_c$ about $\bPi\bx$, gives
\begin{equation*}
\nabla g_c(\bx_c) = \nabla g_c(\bPi \bx)+\sum\limits_{|\balpha|\ge 1} (\bx_c - \bPi \bx)^{\balpha} \partial^{\balpha} \nabla g_c(\bPi\bx).
\end{equation*}
Above (and in what follows), we use multi-index notation, whereby for any $\balpha = (\alpha_i)^d_{i=1} \in {\mathbb R}^d$ and any vector $\by=(y_i)^d_{i=1} \in {\mathbb R}^d$, $\by^{\balpha}$ is shorthand for
$\prod\limits_{i:\;1\le i \le d} y^{\alpha_i}_i$. Also, for a scalar function $f = f(\bx)$ of several variables $\bx =(x_i)^d_{i=1} \in {\mathbb R}^d$, we denote $\partial_i f = \frac{\partial f}{\partial x_i}$, and by $\partial^{\balpha}f(\bx)$ the repeated partial derivatives of corresponding order $\alpha_i$  with respect to the $i$th variable $x_i$ of $f$. 

Combining the last two equations above, we obtain
\begin{equation}\label{FAS correction property}
\bR \nabla g(\bx) +\sum\limits_{|\balpha|\ge 1} (\bx_c - \bPi \bx)^{\balpha} \partial^{\balpha} \nabla g_c(\bPi\bx) = 0.
\end{equation}

\subsection{The FAS-correction}
Once we have computed $\bx_c$, we update the new fine approximation 
$\bx_{new} = \bx+ \bP (\bx_c-\bPi \bx)$. 
Our goal is to see how close the updated $\bx_{new}$ is to the exact solution, or rather how close 
$\nabla g (\bx_{new})$ is to zero. \
For that purpose, we expand $\nabla g (\bx_{new})$ about 
$\bx$ and multiply by $\bR$ on the left which gives
\begin{equation}\label{expansion for FAS correction}
\bR \nabla g (\bx_{new}) = 
\bR \nabla g (\bx) +\sum\limits_{|\balpha|\ge 1}
(\bP(\bx_c-\bPi\bx))^{\balpha} \bR\partial^{\balpha} \nabla g(\bx).
\end{equation}
If we want to make the first order terms of \eqref{expansion for FAS correction} and 
\eqref{FAS correction property} equal, we obtain
\begin{equation*}
\sum\limits_{|\balpha|= 1} (\bx_c - \bPi \bx)^{\balpha} \partial^{\balpha} \nabla g_c(\bPi\bx) =
\sum\limits_{|\balpha|=1}
(\bP(\bx_c-\bPi\bx))^{\balpha} \bR\partial^{\balpha} \nabla g(\bx),
\end{equation*}
which one can rewrite as
\begin{equation}\label{complicated identity}
\sum\limits^{n_c}_{i=1} (\bx_c - \bPi \bx)_i
\partial_i \nabla g_c(\bPi\bx) =
\sum\limits^n_{i=1}
(\bP(\bx_c-\bPi\bx))_i \bR\partial_i \nabla g(\bx).
\end{equation}
Introducing the Hessian matrices (note that they are symmetric)
\begin{equation*}
H_g(\by) = (\partial_i \partial_j g(\by))^n_{i,j=1}
\text{ and } 
H_{g_c}(\by_c) = (\partial_i \partial_j g_c(\by_c))^{n_c}_{i,j=1},
\end{equation*}
identity \eqref{complicated identity} can be written more concisely as
\begin{equation*}
H_{g_c}(\bPi\bx) (\bx_c-\bPi\bx) =
\bR H_g(\bx) \bP(\bx_c-\bPi \bx).
\end{equation*}
In other words, we have proved the following result.
\begin{theorem}\label{FAS theorem}
Let $\bR,\bP,\bPi$ and $g$, $g_c$ satisfy the identity
\begin{equation}\label{FAS identity}
\bR H_g(\bx) \bP = H_{g_c}(\bPi\bx).
\end{equation}
Then the updated approximation $\bx_{new}$ by the FAS algorithm satisfies the approximate first order optimality condition 
\begin{equation*}
\bR \nabla g (\bx_{new})  = 0.
\end{equation*} \qed
\end{theorem}

To satisfy condition \eqref{FAS identity}, since $H_{g_c}$ is a symmetric matrix, we require
\begin{equation}\label{R is P transpose}
\bR = \bP^T.
\end{equation}

For the application of our main interest (training of DNNs), we do not
form the mappings $\bP, \bR, \bPi$ as matrices, but will instead
derive methods for computing the action of these operators. Also, in
general we do not choose $\bR = \bPi$. Instead, we choose $\bPi$ and
$\bP$ such that the product $\bQ = \bP\bPi$ admits an approximation
property (that we motivate in Section~\ref{the choice of g_c}) and
such that
\begin{equation}\label{coarse level projection}
\bPi \bP = I,
\end{equation}
which implies that $\bQ$ is a projection ($\bQ^2 = \bQ$).

\begin{remark}\label{remark: smoothing}
The above theorem shows that the gradient of the functional $g$ is orthogonal (up to first order terms)  to the coarse space $\text{Range}(\bR^T) = \text{Range}(\bP)$. 
Hence, what remains to be done, is to reduce the error in a  space complementary to the coarse space, which is customarily referred to as \emph{smoothing}. In other words, we are looking at $\bx^\perp$ such that
\begin{equation*}
\nabla g (\bx_{new}+\bx^{\perp})  = 0.
\end{equation*}
 The two-level FAS algorithm alternates between a smoothing step and a coarse level solution step as described in detail in Section~\ref{DNN FAS algorithm}. 
\end{remark}

\subsection{Motivating guidelines for the design of the coarse functional $g_c$}\label{the choice of g_c}

One choice is to define
\begin{equation}\label{exact choice}
g_c(\bx_c) = g(\bP\bx_c),
\end{equation}
which admits excellent approximation but is non-optimal in terms of
cost. Let us examine the choice \eqref{exact choice} to motivate an approximation. Assume that $\bP$ and $\bPi$ are such that $\bP\bPi \bx$
approximates $\bx$, i.e., $\|\bx - \bP \bPi \bx\|$ is ``{small}'' in
some norm $\|.\|$.  Let $\bP$, represented as a matrix, admit the form
$\bP= (p_{i,i_c})$, $1\le i \le n$, $1 \le i_c\le n_c$.

Under the coarse functional choice \eqref{exact choice}, we have 
\begin{equation*}
\partial_{j_c} g_c(\bx_c) = \sum\limits^n_{j=1} \partial_j g(\bP\bx_c) p_{j,j_c},
\end{equation*}
and similarly
\begin{equation*}
\partial_{i_c}\partial_{j_c} g_c(\bx_c)= \sum\limits^n_{i,j=1} \partial_i \partial_j g(\bP\bx_c) p_{j,j_c}p_{i,i_c}.
\end{equation*}
In other words, we have
\begin{equation*}
H_{g_c}(\bx_c) = \bP^TH_g(\bP\bx_c)\bP.
\end{equation*}
Now, suppose $\bx_c$ is derived from the restriction $\bx_c = \bPi \bx$. Then, we have
\begin{equation*}
H_{g_c}(\bPi\bx) = \bP^TH_g(\bP\bPi \bx) \bP = \bP^T H_g(\bx)\bP  + \bP^T\left (H_g(\bP\bPi\bx) - H_g(\bx)\right )\bP.
\end{equation*}
Because $\|\bx - \bP \bPi \bx\|$ is small, the second term in this formula is small, so
\begin{equation*}
H_{g_c}(\bPi\bx) \approx \bP^T H_g(\bx) \bP,
\end{equation*}
so that condition \eqref{FAS identity} is approximately
satisfied.

The exact choice of $g_c$ is problem-dependent, and we design such a coarse functional for neural networks below, but this motivates some guidelines:
\begin{enumerate}
    \item[(A)] It is desirable that $g_c$ be chosen so that $H_{g_c}(\bx_c) \approx \bP^TH_g(\bP\bx_c)\bP$
    \item[(B)] It is desirable that $\bR = \bP^T$.
    \item[(C)] It is desirable that $\bP$ and $\bPi$ be chosen so that $\|\bP^T\left (H_g(\bP\bPi\bx) - H_g(\bx)\right )\bP\|$ is small. As the Hessian matrix is a continuous function of $\bx$, this guideline can be simplified to a desire that $\|\bx - \bP \bPi \bx\|_2$ be small.
\end{enumerate}
These properties are desirable because, when they are true, Theorem \ref{FAS theorem} approximately holds, ensuring an approximation for first-order consistency.

\section{Application of Full Approximation Scheme to Deep Neural Network training}\label{FAS applied to DNN training}

In the previous section, we derived some useful properties of two-level FAS for optimization along with a set of advice to help guide the development of
$\bP$ and $\bPi$ operators. In this section, we apply those results to
the problem of training deep neural networks via FAS.

A feedforward DNN (deep neural network) is defined as follows:
\begin{itemize}
\item A set of $n_L$ layers, where the $k$th layer is associated with a vector variable
$\by_k$. 
\item Two consecutive layers are connected by an activation function 
$\varphi$ which acts  on each component of the vectors involved.
\item Also, two consecutive layers, $k$ and $k+1$,  are connected by a transition coefficient matrix
$W_k$, and a shift vector $\bb_k$.
\item To compute the parameters $\{W_k\}^{n_L}_{k=0}$ and $\{\bb_k\}^{n_L}_{k=0}$, the DNN makes use of two sets of given quantities, $\{\by^{in}_s\}$ - input, and $\{\by^{out}_s\}$- the respective output.
\end{itemize}
The DNN is \emph{trained} by computing the weights $\{W_k\}$ and $\{\bb_k\}$, such that 
for the set of given  inputs $\{\by^{in}_s\}$, it produces an accurate approximations to the respective given set of outputs $\{\by^{out}_s\}$, i.e., for each $s$, $DNN(\by^{in}_s) \approx \by^{out}_s$. The DNN parameters $\{W_k\}$ and $\{\bb_k\}$, are computed by solving a suitable optimization problem. 

For any given set of parameters $W = \{W_k\}$ and $\bb = \{\bb_k\}$, let  $\bx$ be a vector quantity that unrolls and concatenates them, that is, let
\begin{equation}\label{unrolled}
\bx = \left [W_0[:], W_1[:], \dots,\; W_{n_L}[:], \bb_0[:],
    \bb_1[:],\dots, \bb_{n_L}[:] \right ]
\end{equation}
be the  unrolled vector of all learnable
parameters in the DNN. Given $\bx$, the DNN acts on any input vector data $\by^{in}$,  i.e., computes 
$\by^{in} \mapsto \by^{out} := DNN(\bx;\;\by^{in})$, by evaluating the recurrence relation
\begin{equation}\label{DNN recurrence}
\begin{array}{rl}
\by_1 & = \varphi\left (W_0 \by^{in} + \bb_0\right ),\\
\by_{k+1} & = \varphi\left (W_k \by_k + \bb_k \right ),\; \text{ for } 0\le k \le n_L-1,\\
\by^{out} & = \varphi\left (W_{n_L} \by_{n_L}+ \bb_{n_L} \right ).
\end{array}
\end{equation}

\begin{remark}
The above definition extends to convolutional neural networks (CNNs)
as well, using Toeplitz structured $W_k$ matrices and $b_k$
vectors. We discuss this and the application of FAS to CNNs in
Section~\ref{convolutional extension}.
\end{remark}

Training the network selects the weight matrices $W= \{W_k\}$ and bias
vectors $\bb = \{\bb_k\}$ by approximately solving the optimization
problem

\begin{equation}\label{eq:nn_obj}
  g(\bx) = g(W, \bb) = \sum_{s=1}^N L_s(DNN(\bx;\;\by^{in}_s)),
  \end{equation}
where $L_s$ is some loss function associated with the $s$th training sample.

One example of  a \emph{loss functional} is the mean-squared error
\begin{equation}\label{mean squared  loss functional}
g(\bx) = \frac{1}{N}\sum\limits^N_{s=1} \|\by^{out}_s - DNN(\bx;\; \by^{in}_s)\|^2.
\end{equation}
In the present setting, to train the DNN means to find a set of parameters $\bx = (W,\;\bb)$ that approximately minimize the above functional $g(\bx)$. However, as discussed in Section \ref{sec:training_methods_for_nn}, one does not generally wish to exactly minimize this functional due to overfitting concerns. Rather, one attempts to find a set of parameters that approximately minimize $g(\bx)$, and achieve a small but nonzero gradient.

The popularity of the DNNs is due to the fact that the gradients of
the loss functional $\nabla g(\bx)$ are readily computable by a \emph{backpropagation}, and a number of popular libraries such as
TensorFlow \cite{tensorflow2015-whitepaper} and PyTorch\cite{pytorch_neurips2019_9015} readily compute them.  Moreover,
in the optimization algorithms employed in the DNN training we may use
only a portion of the sum in \eqref{mean squared loss functional} (or
\eqref{eq:nn_obj}) by employing randomness. That is, for a \emph{minibatch}
of $m$ randomly selected indices $\bs = (s_1,s_2,\;\dots,\;s_m)$, we
compute the gradient of
\begin{equation}\label{batch functional}
g_{\bs}(\bx) = \frac{1}{m}\sum\limits^m_{j=1} L_{s_j}(DNN(\bx;\;\by^{in}_{s_j})).
\end{equation}
This partial (stochastic) gradient, $\nabla g_{\bs}(\bx)$,  is used to update  the iterates $\bx:= \bx - \tau \nabla g_{\bs}(\bx)$ for a chosen step length $\tau \in (0,1]$,   in a gradient descent algorithm to minimize $g(\bx)$.

\section{Restriction, Interpolation and $\tau$-correction operators for the DNN training}
\label{FAS restriction and prolongation}

In this section, we motivate and define our restriction and
interpolation operators, as well as the $\tau$-correction for the case
of training DNNs, in a manner consistent with the guidelines of Section \ref{the choice of g_c}.

First, we clarify notation. Consider a DNN with the vector quantity
$\bx$ containing its learnable parameters as in
Equation~\eqref{unrolled}.  To define a two-level FAS algorithm we
need an interpolation mapping $\bP$ and a fine-to-coarse restriction
mapping $\bPi$ such that $\bQ = \bP\bPi$ is a projection. We will
choose $\bP$ and $\bPi$ such that $\bPi \bP = I$.

As the parameters of a DNN are defined layerwise, we will also
decompose the $\bP$ and $\bPi$ operators in this manner. Consider the
sets $P = \{P_k\}^{n_L+1}_{k=0}$ and $\pi = \{\pi_k\}^{n_L+1}_{k=0}$,
where $P_k$ and $\pi_k$ act on vectors associated with the $k$th
layer; that is, given a data vector $\by_k$ at layer $k$, $\pi_k
\by_k$ converts $\by_k$ to a coarsened $\by_k^c$. We set $\pi_{in} =
\pi_0$, $P_{in} = P_0$ and $\pi_{out} = \pi_{n_L+1}, P_{out} =
P_{n_L+1}$.  We also assume that for each $k$, $\pi_k$ and $P_k$ make a
projection, that is,
\begin{equation}\label{layer projection property}
\pi_k P_k = I.
\end{equation}

Given a set of coarse parameters $\{W^c_k\}$ and $\{\bb^c_k\}$, neural networks in fact provide a natural way to produce a coarse functional $g_c$ that satisfies Guideline (A) of Section \ref{the choice of g_c}. We simply define a coarse DNN as a coarse version of the same recursion as in~\eqref{DNN recurrence}:
\begin{equation}\label{coarse DNN recurrence}
\begin{array}{rl}
\by^c_1 & = \varphi\left (W^c_0 \by^{in}_c + \bb^c_0\right ),\\
\by^c_{k+1} & = \varphi\left (W^c_k \by^c_k + \bb^c_k \right ),\; \text{ for } 0\le k \le n_L-1,\\
\by^{out}_c & = \varphi\left (W^c_{n_L} \by_{n_L}+ \bb^c_{n_L} \right ).
\end{array}
\end{equation}
That is, we have $\by^{out}_c = DNN_c(\bx_c;\;\by^{in}_c)$, where $\bx_c$ combines all parameters $\{W^c_k\}$ and $\{\bb^c_k\}$ as in Equation~\eqref{unrolled}.

The training objective function is a summation over a coarse set of training samples:

\begin{equation}\label{eq:coarse_nn_obj}
  g_c(\bx_c) = g_c(W^c, b^c) = \sum_{s=1}^N L^c_s(DNN_c(\bx_c;\; \by^{in}_{c,s})).
  \end{equation}
If we do not have coarse counterparts of the training samples $\by^{in}_c$ and $\by^{out}_c$, we use the only given (fine level) set of samples, i.e., we let
$\by^{in}_{c,s} = \by^{in}_s$ and $\by^{out}_{c,s} = \by^{out}_s$. This implies that we let $\pi_0 = I, P_0 = I$ and $\pi_{n_L+1} =I, P_{n_L+1} = I$.
Also, in that case we have $L^c_s = L_s$.  Because $g_c$ is designed to approximate $g$ here, Guideline (A) is satsified.

In what follows, we construct a coarse DNN from a fine DNN parameters; and vice versa, given  a coarse DNN, we embed it in a finer  DNN.

\subsection{Fine-to-coarse DNN restriction}

Let us assume that $\|\by_k - P_k \pi_k \by_k \|$ is small in the
sense discussed in Section \ref{the choice of g_c}. Then the following
is an approximation to the standard DNN recurrence~\eqref{DNN
  recurrence}:
\begin{equation}\label{DNN recurrence in subspace}
\begin{array}{rl}
\pi_1 \by_1 & = \pi_1 \varphi\left (W_0 P_0 \pi_{in}  \by^{in} + \bb_0\right ),\\
\pi_{k+1} \by_{k+1} & = \pi_{k+1}\varphi\left (W_k P_k \pi_k \by_k + \bb_k \right ),\; \text{ for } 0\le k \le n_L-1,\\
\pi_{out} \by^{out} & = \pi_{out} \varphi\left (W_{n_L} P_{n_L} \pi_{n_L} \by_{n_L}+ \bb_{n_L} \right ).
\end{array}
\end{equation}
By letting $\by^c_k = \pi_k  \by_k$,  we have the following recurrence involving coarse vectors $\by^c_k$,
\begin{equation}\label{coarse nonlinear RAP DNN recurrence}
\begin{array}{rl}
\by^c_1 & = \pi_1 \varphi\left (W_0 P_0  \by^{in}_c + \bb_0\right ),\\
\by^c_{k+1} & = \pi_{k+1}\varphi\left (W_k P_k \by^c_k + \bb_k \right ),\; \text{ for } 0\le k \le n_L-1,\\
\by^{out}_c & = \pi_{n_L+1} \varphi\left (W_{n_L} P_{n_L} \by^c_{n_L}+ \bb_{n_L} \right ).
\end{array}
\end{equation}

Finally, we assume the approximation
\begin{equation}\label{activation function approx with pi}
  \varphi(\pi \bz) \approx \pi \varphi(\bz).
\end{equation}
We discuss in Section~\ref{choice for pi and P} when this is valid. Given this,
recurrence~\eqref{coarse nonlinear RAP DNN recurrence} can be
approximated as
\begin{equation*}
\begin{array}{rl}
\by^c_1 & = \varphi\left (\pi_1 W_0 P_0  \by^{in}_c + \pi_ 1 \bb_0\right ),\\
\by^c_{k+1} & = \varphi\left (\pi_{k+1} W_k P_k \by^c_k + \pi_{k+1}\bb_k \right ),\; \text{ for } 0\le k \le n_L-1,\\
\by^{out}_c & = \varphi\left (\pi_{n_L+1} W_{n_L} P_{n_L} \by^c_{n_L}+ \pi_{n_L} \bb_{n_L} \right ).
\end{array}
\end{equation*}
That is, letting
\begin{equation}\label{coarse DNN parameters}
W^c_k = \pi_{k+1} W_k P_k \text{ and } \bb^c_k = \pi_{k+1}\bb_k,
\end{equation}
we end up with the recurrence that implements the actions of a coarse DNN, $\by^{in}_c \mapsto \by^{out}_c = DNN_c(\bx_c;\; \by^{out}_c)$:
\begin{equation}\label{restricted DNN recurrence}
\begin{array}{rl}
\by^c_1 & = \varphi\left (W^c_0 \by^{in}_c +  \bb^c_0\right ),\\
\by^c_{k+1} & = \varphi\left (W^c_k\by^c_k + \bb^c_k \right ),\; \text{ for } 0\le k \le n_L-1,\\
\by^{out}_c & = \varphi\left (W^c_{n_L}  \by^c_{n_L}+  \bb^c_{n_L} \right ).
\end{array}
\end{equation}

This motivates our fine-to-coarse mapping, $\bPi$, which we define as
\begin{equation}\label{fine-to-coarse transfer mapping}
  \begin{array}{rl}
    \bx = & [(W_k),\; (\bb_k)], \\
    \bPi \bx = & [(\pi_{k+1} W_k P_k),\; (\pi_{k+1}\bb_k)].
  \end{array}
\end{equation}
Note that, despite the layer-wise decomposition, this is still a
linear mapping of $\bx$, so there exists a matrix $\bPi$ whose action
is this.

\subsection{Coarse-to-fine DNN prolongation}
Consider the coarse DNN recurrence \eqref{restricted DNN recurrence}. We assume the approximation
\begin{equation}\label{activation function approx with P}
  P \varphi(\bz_c) \approx \varphi(P \bz_c),
\end{equation}
discussed in Section~\ref{choice for pi and P}, and using the property $\pi_k P_k = I$, we have
\begin{equation*}
\arraycolsep=1.5pt
\begin{array}{rll}
P_1 \by^c_1 & = P_1 \varphi\left (W^c_0 \pi_0 (P_0 \by^{in}_c) +  \bb^c_0\right ) & \approx \varphi\left (P_1 W^c_0 \pi_0 (P_0 \by^{in}_c) +  P_1\bb^c_0\right ),\\
P_{k+1} \by^c_{k+1} & = P_{k+1} \varphi\left (W^c_k\pi_k (P_k \by^c_k) + \bb^c_k \right )& \approx
\varphi\left (P_{k+1} W^c_k\pi_k (P_k \by^c_k) +P_{k+1} \bb^c_k \right ),\;  k \le n_L-1,\\
P_{out} \by^{out}_c & = P_{out} \varphi\left (W^c_{n_L} \pi_{n_L} (P_{n_L} \by^c_{n_L})+  \bb^c_{n_L} \right ) & \approx
\varphi\left (P_{n_L+1} W^c_{n_L} \pi_{n_L} (P_{n_L} \by^c_{n_L})+  P_{n_L+1} \bb^c_{n_L} \right ).
\end{array}
\end{equation*}
The latter formulas motivates us to define the interpolation mapping $\bP$ as follows.
Let $\bx_c = [(W^c_k),\; (\bb^c_k)]$, then
\begin{equation}\label{interpolation mapping}
\bP\bx_c = [(P_{k+1}W^c_k \pi_k),\; (P_{k+1}\bb^c_k)].
\end{equation}

Note that this does not require that the $W_k$ matrices be square. It only requires that the input dimension to layer $k$ match the output dimension of layer $k-1$.

\subsection{The restriction mapping $\bR= \bP^T$}

The definition of $\bR$ is already given by~\eqref{R is P
  transpose}. The layerwise definition is
\begin{equation}\label{restriction mapping}
\bR \bx = [(P^T_{k+1} W_k \pi^T_k),\; (P^T_k \bb_k)].
\end{equation}
The following result, satisfying Guideline (B) of Section \ref{the choice of g_c}, holds.
\begin{lemma}\label{lemma: R = transpose of P}
The restriction mapping $\bR$ defined in \eqref{restriction mapping} and the interpolation mapping $\bP$ defined in \eqref{interpolation mapping} are transpose to each other, i.e.,
\begin{equation*} \bR = \bP^T.
\end{equation*}
\end{lemma}
\begin{proof}
We start with the inner product of $\bx = [(W_k),\;(\bb_k)]$ and $\bP \bx_c = \bP [(W^c_k),\; (\bb^c_k)]$, 
\begin{equation*}
(\bx_c)^T (\bP^T \bx) = \bx^T (\bP \bx_c) = \sum\limits_i x_i (\bP \bx_c)_i.
\end{equation*}
With $\bb_k = (b_{k,i})$, $\bb^c_k = (b^c_{k,i_c})$,  $W_k = \left ((W_k)_{r,s}\right )$ and $W^c_k = \left ((W^c_k)_{r_c,s_c} \right )$, based on the Definition~\eqref{interpolation mapping}, we have the representation 
\begin{equation}\label{term i plus ii}
\bx^T (\bP \bx_c)  = \sum\limits_i x_i (\bP \bx_c)_i = \sum\limits_k \sum\limits_i b_{k,i} (P_k \bb^c_k)_i + \sum\limits_k \sum\limits_{i=(r,s)} (W_k)_{r,s} (P_{k+1} W^c_k \pi_k)_{r,s}.
\end{equation}
The first term on the right gives
\begin{equation}\label{term i}
\sum\limits_k \sum\limits_i b_{k,i} (P_k \bb^c_k)_i  = \sum\limits_k (\bb^c_k)^T P^T_k \bb_k= \sum\limits_k (\bb^c_k)^T \bb^c_k.
\end{equation}
For the second term, by rearranging terms and  using the definition \eqref{restriction mapping} of $\bR \bx$, 
we have
\begin{equation}\label{term ii}
\begin{array}{rl}
 \sum\limits_k \sum\limits_{i=(r,s)} (W_k)_{r,s} (P_{k+1} W^c_k \pi_k)_{r,s} & = \sum\limits_k  \sum\limits_{i=(r,s)} \sum\limits_{i_c= (r_c,s_c)}
 (W_k)_{r,s} (P_{k+1})_{r,r_c}  (W^c_k)_{r_c,s_c} (\pi_k)_{s_c,s} \\
& = \sum\limits_k \sum\limits_{i_c=(r_c,s_c)}  (W^c_k)_{i_c} \sum\limits_{r,s}   (P_{k+1})_{r,r_c}(W_k)_{r,s}  (\pi_k)_{s_c,s}\\
& = \sum\limits_k \sum\limits_{i_c=(r_c,s_c)}   (W^c_k)_{i_c} (P^T_{k+1} W_k \pi^T_k)_{i_c}\\
& = \sum\limits_k \sum\limits_{i_c} (W^c_k)_{i_c} (W^c_k)_{i_c}.
\end{array}
\end{equation}
Thus substituting \eqref{term i} and \eqref{term ii} in \eqref{term i plus ii}, using formulas \eqref{restriction mapping} defining the components of $\bR \bx$, we arrive at 
\begin{equation*}
(\bx_c)^T (\bP^T \bx)  = \bx^T (\bP \bx_c)   = (\bx_c)^T \bR \bx,
\end{equation*}
which proves the desired result.
\end{proof}

\begin{remark}\label{remark: P, Pi and R are mappings}
We note that the linear mappings $\bPi$, $\bP$, as well as $\bR= \bP^T$ are not explicitly formed as matrices. However, from their definitions it is clear that their actions are readily available
via the actions of the layer matrices $P_k$, $\pi_k$ and their transposes. 
That is, in an implementation, only the layer matrices $\pi_k$ and $P_k$ need to be  constructed and stored explicitly as (sparse) matrices.
\end{remark}

\begin{lemma}\label{projection property of pi p}
With the definitions \eqref{fine-to-coarse transfer mapping} of $\bPi$ and \eqref{interpolation mapping} of $\bP$, assuming \eqref{layer projection property} for the layer matrices $\pi_k$ and $P_k$,  the following main property holds
\begin{equation*}
\bPi \bP = I,
\end{equation*}
which implies that $\bQ = \bP \bPi$ is a projection ($\bQ^2 = \bQ$). 
\end{lemma}
\begin{proof}
Given $\bx_c = [(W^c_k),\; (\bb^c_k)]$, define $\bx = [(W_k),\;(\bb_k)] = \bP \bx_c$, i.e., $W_k = P_{k+1}W^c_k \pi_k$ and $\bb_k =  P_{k+1}\bb^c_k$. 
Then for $\bPi \bx = [(\pi_{k+1} W_k P_k),\; (\pi_{k+1}\bb_k)]$ since $\pi_j P_j = I$, we have 
\begin{equation*}
\begin{array}{rl}
\pi_{k+1} W_k P_k &= \pi_{k+1}  P_{k+1}W^c_k \pi_kP_k = W^c_k,\\
\pi_{k+1}\bb_k & = \pi_{k+1}P_{k+1}\bb^c_k = \bb^c_k.
\end{array}
\end{equation*}
That is, $\bPi\bx = \bPi \bP \bx_c = \bx_c$, which  completes the proof.
\end{proof}

\subsection{A simple choice for $\pi_k$ and $P_k$}\label{choice for pi and P}
Finally, we must choose $\bPi$ and $\bP$ to satisfy Guideline (C) of Section \ref{the choice of g_c}. Intuitively, if two neurons have nearly the same weights and biases, then we could average them together to construct a coarse neuron that approximates the two fine neurons. Therefore, we choose $\pi_k$ to perform pairwise averaging and $P_k$ to perform piecewise constant interpolation over selected pairwise aggregates.
The choice of the pairwise aggregates is done through a Heavy Edge Matching (HEM) method \cite{metis}.
The mappings are simply written as follows
\begin{equation}\label{pi and P}
P_k = \begin{bmatrix}
1 \\
1 \\
 & 1 \\
 & 1 \\
  & & \ddots
\end{bmatrix} \in \mathbb{R}^{n_k \times n^c_k},\quad 
\pi_k = \begin{bmatrix}
1/2 & 1/2 \\
    &    & 1/2 & 1/2 \\
 & &  & & \ddots
\end{bmatrix} \in \mathbb{R}^{n^c_k \times n_k}.
\end{equation}

When the activation function $\varphi$ is the commonly-used ReLU function,
\begin{equation}\label{relu}
  \varphi(x) = \left\{
  \begin{array}{rl}
    x, & x \geq 0 \\
    0, & x < 0
  \end{array} \right.,
\end{equation}
these choices for $P_k$ and $\pi_k$ justify the approximations
\eqref{activation function approx with pi} and \eqref{activation
  function approx with P}. In particular, \eqref{activation function
  approx with pi} is valid because the HEM (discussed in Section \ref{sec:HEM} and shown in Algorithm \ref{alg:greedyhem}) matches together neurons whose parameter vectors point in approximately the same direction, and as a result, tend to produce positive and negative dot products with input vectors at the same time. The nonlinearity of the ReLU depends entirely on that positive/negative property of its input, and so \eqref{activation function approx with pi} will be exactly true except when dot products with paired neurons produce different signs. Similarly, \eqref{activation function approx with P} is in fact an equality, by design, due to the piecewise-constant prolongation. Other activation functions are likely effective here as well; in general, if we approximate $\varphi(\bz)$ near $\bz$ linearly with a \emph{matrix} $\widehat \varphi$, then \eqref{activation function approx with P} holds when $\pi \widehat \varphi \approx \widehat \varphi \pi$. In particular, other piecewise-linear activation functions such as the "leaky ReLU" satisfy this.

A variant to \eqref{pi and P} that is often found better is to
scale $P_k$ by the row norms of $W_k$, i.e.,
\begin{equation}\label{weighted pi and P}
    \tilde P_k = D_k P_k, \quad \tilde \pi_k = (P_k^T D_k P_k )^{-1} P_k^T,
\end{equation}
where $D_k$ is the diagonal matrix that contains the row norms of
$W_k$, so that the rows that correspond to the paired neurons in an aggregate are invariant
with the projector $Q$
if they are linear dependent,
i.e., they have the maximum similarity score in the HEM algorithm.
Apparently, \eqref{pi and P} is a special case where
$D_k=I$.
Furthermore, $P_k^T D_k P_k$ is a diagonal matrix due to the structures of $P_k$ and $D_k$.

\begin{remark}
Although we write Equations~\eqref{pi and P} and \eqref{weighted pi and P} as matrices, these are very sparse matrices and the \emph{actions} of these matrices can be efficiently computed without ever directly forming $\tilde P_k$ and $\tilde \pi_k$. In our software, we have implemented matrix-free versions for fast, practical computation.
\end{remark}

\subsection{HEM algorithm for coarsening}\label{sec:HEM}
Standard HEM algorithms have been used as coarsening algorithms for aggregation-based AMG,
see, e.g., \cite{aggamg}.
Finding the optimal HEM of a graph is in general difficult, 
but greedy algorithms can usually work well.
A sketch of greedy HEM is given in Algorithm~\ref{alg:greedyhem},
where  input $S_k$ is the strength-of-connection matrix, 
which defines the similarities between neurons.
Cosine angles of the rows of $W_k$ are used for $S_k$.
In Algorithm~\ref{alg:greedyhem}, neurons are visited in a given order and at each time 
if neuron $i$ has
not been  matched, neuron $j$ that has  
the heaviest edge between $i$ and $j$
among all the remaining unmatched neurons is selected.
Edges that have weights smaller than a predefined threshold $\theta$ are  disregarded in the matching.
The results are kept in $\pi$ and $\zeta$, where
$\pi(i)$ keeps the matched neuron of $i$, and
$\zeta(i)$ stores
the index of the aggregate of $i$ and $\pi(i)$.
Clearly, on exit of the algorithm, we have $\pi(\pi(i))=i$ (if $i=\pi(i)$, $i$ is 
a singleton), and $\zeta(i)=\zeta(\pi(i))$.
Roughly speaking, a coarsening factor close to 2 is typically obtained. The HEM method can also be parallelized for faster processing on GPUs via parallel matching methods such as \cite{naumov-matching}.

\begin{algorithm}[ht]
\caption {A greedy HEM algorithm}
\label{alg:greedyhem}
\begin{algorithmic}[1]
\State \textbf{Input:} Strength-of-connection matrix $S_k$ and strength threshold $\theta$
\State \textbf{Output:} matching $\pi$, aggregation $\zeta$, and the number of aggregates $n_k^c$
\For{$i=1,2,\ldots,n_k$}
\State $\pi(i) := 0$
\EndFor
\State $n_k^c := 0$ 
\For{$i=1,2,\ldots,n_k$}
\If{$\pi(i)=0$}
\If{$\exists j \neq i : \pi(j)=0 \land S_{ij} > \theta \land \forall k : S_{ij} \ge S_{ik}, \pi(k)=0$}
\State $\pi(i) := j, \; \pi(j) := i$
\State $n_k^c := n_k^c + 1, \; \zeta(i) := n_k^c, \; \zeta(j) := n_k^c$
\Else \Comment{match $i$ to itself}
\State $\pi(i) := i, \; n_k^c := n_k^c + 1, \; \zeta(i) := n_k^c$
\EndIf
\EndIf
\EndFor
\end{algorithmic}
\end{algorithm} 

\subsection{Constructing a tau-correction}\label{sec:tau_construction}

As discussed above, the tau correction in optimization is $\tau =
\nabla g_c(\bPi \bx) - \bR \nabla g(x)$, where $\bR= \bP^T$, and $\bPi$ are the  two linear fine-to-coarse mappings,\eqref{restriction mapping}-\eqref{fine-to-coarse transfer mapping}, defined in our DNN setting.
 Note that the operation $\bR \nabla g(\bx)$ is
well-defined, as $\nabla g(\bx)$ has the same structure as $\bx$.

One wrinkle is that the calculation of a gradient requires a summation
over all training samples in the training set. Often, training sets
have a large number of samples in them, so computing the true gradient
would be computationally unwieldy. We approximate this by simply
calculating a stochastic tau-correction, in which we average over a
small subset $\bs = (s_1,\;s_2,\;\dots,\;s_m)$  of the total training set, and then scale appropriately:

\begin{equation}\label{tau}
\begin{array}{rl}
\nabla_{\bx}g_\bs (\bx) &\displaystyle = \sum_{j=1}^{m} \nabla_{\bx} L_{s_j}(DNN(\bx,\; y^{in}_{s_j}) \\
\nabla_{\bx_c}g^c_\bs(\bx_c) &\displaystyle = \sum_{j=1}^{m} \nabla_{\bx_c} L^c_{s_j}(DNN_c(\bx_c,y^{in}_{c,s_j}) \\
\btau &\displaystyle = \frac{N}{m} \left (\left (\nabla_{\bx_c}  g^c_{\bs}\right )(\bPi \bx) - \bR \nabla_{\bx}g_{\bs} (\bx) \right ),
\end{array}
\end{equation}
where $m$ is the size of a tau batch, $s_j$ is the index of the
$j$'th sample in the tau batch, $\nabla_\by$ stands for a gradient with respect to variable $\by$, and $N$ is the number of minibatches in the dataset as a whole. 

Due to this stochastic approximation, we in general will never construct a coarse problem such that $\nabla \widehat g(\bx_c) = 0$ exactly. This is consistent with how neural networks are trained: Specifically, because traditional training is performed in a stochastic manner over minibatches, exact gradients are never computed, and one never achieves an $\bx$ such that $\nabla g(\bx) = 0$ exactly, or even iteratively converges to such a point. In fact, this behavior is often considered desirable, see Section \ref{sec:training_methods_for_nn}.

\section{The DNN FAS Algorithm}\label{DNN FAS algorithm}

The original objective function $g$ is defined in Equation \eqref{eq:nn_obj}, and the
auxiliary objective function in Equation \eqref{eq:coarse_nn_obj}. Given a tau correction, we can
construct a tau-corrected  auxiliary problem like in Equation \eqref{eq:tau_corrected_aux}. 

We recall that the FAS step cannot be our only method for improving
$\bx$.  As discussed earlier, Remark~\ref{remark: smoothing}, we need
a correction in a space complementary to the coarse space
$\text{Range}(\bP)$, a step customarily referred to as
\emph{smoothing}.  That is because $\bx_c$ exists in a
reduced-dimension space and so cannot represent all the possible
changes we might make.

There exist a wide variety of methods one might choose for the
smoothing step. In general, it should be cost-efficient and
effectively target the complement to $\text{Range}(\bP)$. The neural network community has already provided a number of efficient first-order training methods which can be leveraged for smoothing. We choose here to use stochastic gradient descent (SGD). Our reason for using this is that SGD in this context is quite analogous to the commonly-used Jacobi smoother in traditional multigrid methods. Indeed, when applying the Full Approximation Scheme to the deterministic problem of minimizing quadratic forms, (deterministic) gradient descent is exactly a Richardson iteration, very similar to the Jacobi iteration commonly used as a multigrid smoother \cite{briggs-multigrid}. A two-level FAS V-cycle, then, consists of the following steps:

\begin{enumerate}
\item Apply one or more iterations of SGD to the original problem $g(\bx) = g(W,\bb)$.
  
\item Calculate the stochastic tau correction $\btau$.
  
\item Apply the restriction operation to calculate the elements of the coarse neural network $W^c, \bb^c$.
  
\item Apply one or more iterations of SGD to the corrected auxiliary problem ${\widehat g}(\bx_c) = {\widehat g}(W^c,\bb^c)$ (see \eqref{eq:tau_corrected_aux}).
  
\item Apply the prolongation operation $\bP(\bx_c - \bPi \bx)$ and then the FAS update $\bx \leftarrow \bx_{new} = \bx + \bP(\bx_c - \bPi \bx)$.
  
\item Apply one or more iterations of SGD to the original problem $g(\bx) = g(W,b)$.
 
\end{enumerate}

FAS V-cycles of greater than two levels are achieved through recursion at Step 4.

\subsection{A Few Notes on Algorithmic Tuning}

\subsubsection{Minibatch Selection}

We perform minibatch selection in a manner intended to closely match that of traditional one-level training. We first shuffle the entire training dataset, and then proceed by cycling through the dataset. When computing the stochastic tau correction, we draw the next $m$ minibatches in shuffled order, compute $\btau$ over those $m$ minibatches as in Equation \eqref{tau}, and then use those same $m$ minibatches for SGD at the coarse level. A new set of $m$ minibatches are drawn when restricting to the next-coarser level. Because SGD at the coarse level is performed using the same minibatches as were used for the tau-correction, one might view this as a \emph{stochastic V-cycle}. Because multigrid V-cycles typically only apply a small number of smoothing steps at each level, we tend to keep $m$ fairly small as well, often between 2 and 8.

\subsubsection{Momentum}

Momentum \cite{sutskever-momentum} is an important component in training methods, as it helps to smooth out variations and move parameters more consistently in directions that matter. In fact, SGD with momentum (and many traditional training methods) use gradient information to update momentum vectors, and only directly update parameter vectors through momentum. In our setting, important information about desirable directions of movement is learned at the coarse levels, and we should not throw this important information away upon prolongation. Therefore, we modify this approach slightly to restrict and prolong \emph{momentum} as well as network parameters. As there is a momentum parameter for each network parameter, a momentum vector $\bm$ has the same structure as a parameter vector $\bx$.

Upon restriction, we also restrict the momentum vector using the same restriction operator $\bm_c = \bPi \bm$, and upon prolongation, we perform the same coarse-grid correction
\begin{equation}\label{momentum correction}
\bm = \bm + \bP(\bm_c - \bPi \bm).
\end{equation}
We show the two-level FAS algorithm, with multilevel momentum, in Algorithm~\ref{TL FAS algorithm}.

\subsubsection{HEM Re-Matching}

The HEM matching in Algorithm \ref{alg:greedyhem} depends on the weights, which change during the training process. Therefore, matchings tend to change somewhat throughout the training process. We find that it is generally beneficial to periodically re-compute matchings, but that the training is not particularly sensitive to the rate of re-matching. In our tests, we choose to re-match at a rate that keeps the computational cost of matching minimal compared to the overall training cost, such as every 10 or 50 V-Cycles.

\begin{algorithm}
  \caption{Two-Level Full Approximation Scheme for Deep Neural Networks}\label{TL FAS algorithm}
  \begin{algorithmic}

    \Require{$\bx, g$}

    \While{not converged}
    \State $\bx \gets \text{smooth}(\bx, g)$

    \State Construct $\bPi, \bP$ via Heavy Edge Matching

    \State $\bx_c \gets \bPi \bx, \quad \bm_c \gets \bPi \bm$ \Comment{\eqref{fine-to-coarse transfer mapping}}
  
    \State $\btau \gets \frac{N}{m} \left (\left (\nabla_{\bx_c}  g^c_{\bs}\right )(\bPi \bx) - \bR \nabla_{\bx}g_{\bs} (\bx) \right )$ \Comment{\eqref{tau}}

    \State $\bx_c \gets \text{smooth}(\bx_c, {\widehat g_c}(\bx_c))$, where ${\widehat g_c}(\bx_c) = g_c(\bx_c) - \bx_c^T\btau$ \Comment{\eqref{eq:tau_corrected_aux}}

    \State $\bx \gets \bx + \bP(\bx_c - \bPi \bx), \quad \bm \gets \bm + \bP(\bm_c - \bPi \bm)$ \Comment{\eqref{TL FAS correction}}
  
    \State $\bx \gets \text{smooth}(\bx, g)$

    \EndWhile

\end{algorithmic}
\end{algorithm}

\subsubsection{Maintaining Stability}\label{maintaining stability}

In some cases, this multilevel hierarchies developed here can cause instabilities, and in severe cases, divergence of the method. Here we discuss a number of hyperparameters that can be used in the method to maintain stability when this happens:
\begin{itemize}
    \item The learning rate can be decreased at coarser levels of the hierarchy. We often choose to decrease the learning rate by a fixed factor $\eta$ at each hierarchy level. This can prevent the coarse parameters from moving outside the area in which the coarse problem approximates the fine problem effectively. However, the learning rate does not impact the stored momentum vector \cite{Goodfellow-et-al-2016}, so that must be stabilized separately.
    \item The coarse-grid parameter correction and coarse-grid momentum correction in Algorithm \ref{TL FAS algorithm} and in Equations \eqref{TL FAS correction}, \eqref{momentum correction} can be scaled down with parameters $\alpha_p, \alpha_m \in (0, 1]$ as follows:
    \begin{align*}
        \bx \gets & \bx + \alpha_p \bP (\bx_c - \bPi \bx) \\
        \bm \gets & \bm + \alpha_m \bP (\bm_c - \bPi \bm).
    \end{align*}
    This is akin to the $\alpha$ parameter used in the line search in MG/OPT \cite{Nash2000discretized}; however, we use constants here rather than the line search of MG/OPT due to the added computational cost.
    \item The $\tau$-correction in Algorithm \ref{TL FAS algorithm} and in Equation \eqref{eq:tau_corrected_aux} can be scaled down by $\gamma \in [0, 1]$, so that 
    $$
    {\widehat g_c}(\bx_c) = g_c(\bx_c) - \gamma \bx_c^T\btau.
    $$
    Because our $\tau$-correction is stochastic, even when we are near the optimum, any given set of minibatches could potentially result in a much larger $\btau$, which could then cause an inappropriate coarse-grid correction.
\end{itemize}

We note that in MG/OPT \cite{Nash2000discretized}, it is also suggested one could add a quadratic penalty to the coarse problem to prevent too much movements; however, MG/OPT expects one to fully solve the coarse problem, whereas we only take a few SGD steps. In our experiments, tuning the weight of that quadratic penalty caused its own stability challenges, so we do not use that strategy.

When making these adjustments in Section \ref{experiments}, we typically choose $\eta \in [\sqrt{2}, 2 \sqrt{2}], \alpha_p = 1, \alpha_m = 0.2, \gamma = 0.125$. These were chosen by trial and error.

\section{Extension to Convolutional Neural Networks}\label{convolutional extension}

In this section we describe how the above method can be extended from
fully-connected networks to convolutional neural networks (CNNs). A
CNN typically consists of one or more convolutional layers followed by
one or more fully-connected layers. They are often effective in
learning patterns over array-structured data. In a CNN, the object
analogous to a fully-connected neuron is the \emph{channel};
therefore, our goal in this section will be to develop methods for
coarsening over the channels of a CNN. The resulting convolution
operators will have the same width and stride, but the number of
channels in each layer will be restricted.

Convolutional operators are often represented as higher-order tensors,
but this tensor acts on the inputs in a linear fashion, and so the
action can be represented as matrix-vector multiplication. As a
result, convolutional layers are mathematically quite similar to
fully-connected layers, and the above method extends in a
straightforward manner. That is, the action of a convolutional layer
can still be viewed in the form
$$
\by_{k+1} = \varphi(W_k \by_k + \bb_k),
$$
where $W_k$ and $\bb_k$ have a particular structure.

To apply our restriction and prolongation operators, we require an
understanding of the structure of the matrix $W$ associated with a
convolutional layer. This depends on the dimension of the input array,
the number of input channels, and the number of output channels. We
build up our understanding progressively as follows:

\begin{itemize}
  
\item Consider input $\by$ to a convolutional layer that is a
1-dimensional array and has a single channel; for example, audio data
has this structure. For simplicity, suppose we have a convolutional
layer with a single convolutional channel, which we represent as a
vector $\bc = (c)_{i=-\ell}^\ell$ and a scalar $b$. Then the action of
the convolution is
\begin{equation}
  \sum_{j=-\ell}^\ell y(i+\ell-j) \cdot c(j) + b.
\end{equation}
We can represent this as a matrix-vector multiplication $W \by$ where $W$ is the banded Toeplitz matrix
\begin{equation}
  W = \begin{bmatrix}
    c_{-\ell} & c_{-\ell+1} & \cdots      & c_{k-1}     & c_k     & 0      & \cdots & 0      & 0      & 0 \\
    0         & c_{-\ell}   & c_{-\ell+1} & \cdots      & c_{-1}  & c_k    & 0      & \cdots & 0      & 0 \\
    0         & 0           & c_{-\ell}   & c_{-\ell+1} & \cdots  & c_{-1} & c_k    & 0      & \cdots & 0 \\
    \vdots    &             &             &             &         &        &        &        &        & \vdots
  \end{bmatrix}.
\end{equation}

\item More complex convolutional layers extend this in a natural
way. One-dimensional, single-channel arrays with multiple
convolutional channels results in a matrix
$$
W = [W^1 ; \cdots ; W^p]^T
$$
that is a column of banded Toeplitz blocks.

\item Two-dimensional, single-channel arrays with a single
  convolutional channel produces a weight matrix $W$ that has banded
  block-Toeplitz structure. See Figure~\ref{banded block Toeplitz}.

  \begin{figure}
    \centering
    \includegraphics[width=10cm]{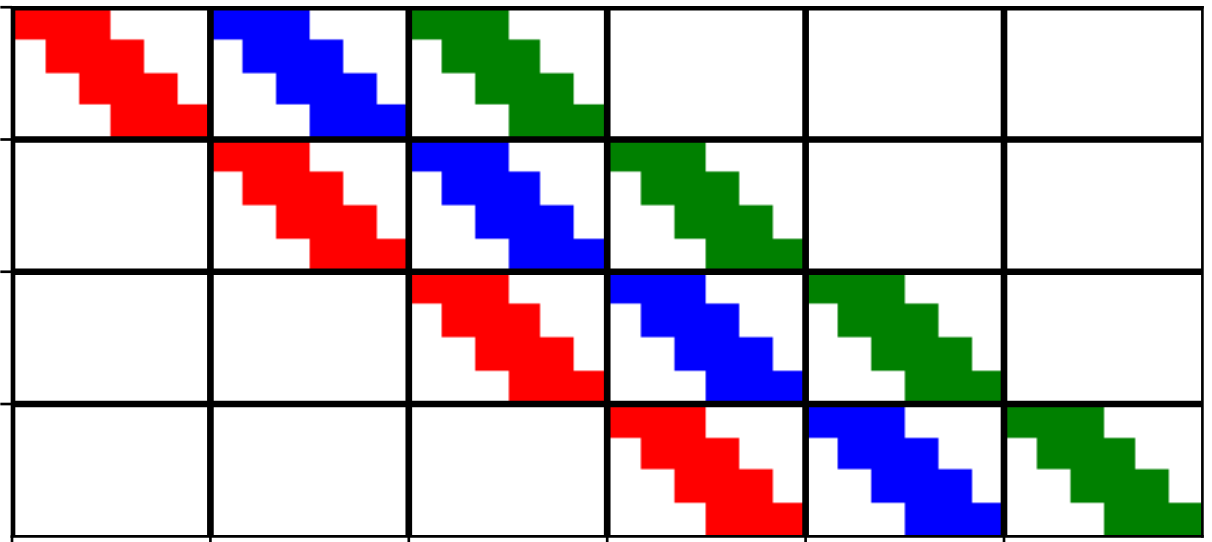}
    \caption{The banded block-Toeplitz structure of a convolutional weight matrix with a single channel applied to a 2D input with a single channel.}
    \label{banded block Toeplitz}
  \end{figure}

\item Adding additional output channels creates a block-column matrix
  in which each block has this structure. Thus, two-dimensional,
  single-channel arrays with multiple convolutional channels produce a
  column of matrix blocks each of the form of Figure~\ref{banded block
    Toeplitz}. We might say this then has \emph{block banded block
    Toeplitz} structure.

\item Finally, multi-dimensional input arrays with multiple inputs
  channels and multiple output channels results in a weight matrix $W$
  of the form
  \begin{equation}
    W = \begin{bmatrix}
      W^{1,1} & W^{1,2} & \cdots & W^{1,\beta} \\
      \cdots  & \vdots  &        & \cdots      \\
      W_{\alpha,1} & W_{\alpha,2} & \cdots & W^{\alpha, \beta}
    \end{bmatrix}
  \end{equation}
  in which each submatrix $W^{i,j}$ is a matrix of the form of
  Figure~\ref{banded block Toeplitz}.

  A key observation here is that every block row corresponds to an
  output channel and every block column corresponds to an input
  channel. As each channel (block row) here is analogous to a neuron
  (individual row) in a fully-connected layer, we see that, analogous
  to Section~\ref{choice for pi and P}, we can coarsen two channels
  together by simply averaging their channels together.

\end{itemize}

Therefore, we see that applying Algorithm~\ref{TL FAS algorithm} to
CNNs proceeds in exactly the same manner. Restrictions occurs through
averaging channels as in $\pi_k$ of Equation~\eqref{pi and P}, and
prolongation occurs through piecewise constant interpolation of
channels as in $P_k$ of Equation~\eqref{pi and P}.

\section{Experiments}\label{experiments}

In this section we demonstrate the use of our multilevel FAS trainer on a set of PDE-based regression problems to compare this to SGD. These experiments were performed on the Lassen compute cluster at Lawrence Livermore National Laboratory. For each experiment, we used a single IBM Power9 nodes with 44 CPU cores at 3.5 Ghz and 256 GB of memory and a single NVIDIA V100 Voltage graphics card. What we find with these results is that in some cases the use of FAS can improve $L_2$ loss, and in all cases it achieves $L_2$ loss at least as good as SGD. In addition, in all cases tested, it \emph{also} improves upon $L_\infty$ loss even though the training was against $L_2$. This implies that FAS-trained neural networks has better worst-case performance than those trained with traditional SGD, suggesting better generalization power.

Note that in these experiments, to achieve a more faithful comparison in computational work, we borrow a strategy from the multigrid literature and report \emph{work units} rather than epochs. A single work unit, as we define it, is the computational cost to compute a forward and backpropagation over one minibatch. We do this because coarser levels of the multilevel hierarchy are less expensive to compute over; as a rule of thumb, half as many neurons in each level implies weight matrices with 1/4 as many elements, and so computation at a coarse level requires roughly 1/4 as much work as the next-finer level. For example, if we perform a V-Cycle over a 2-level hierarchy with 2 minibatch presmoothing iterations, 2 minibatch coarse grid iterations, and then 2 minibatch postsmoothing iterations, that costs 4.5 work units. In addition, computing the tau-correction requires both a fine-level and a coarse-level forward and backpropagation over each minibatch that will be used at the coarse level; in our example, that would add an additional 2.5 work units, for a total of 7 work units for the V-Cycle. In general, if $s$ is the cost of a forward/backward pass, then work units can be computed as
\begin{align}
    & \mbox{(smoothing cost)} + \mbox{(tau correction work for current level)} + \mbox{(tau correction work for next level)} \notag \\
    = & \left(\sum_{i=0}^{k-2} (1/4)^i \cdot 2s + (1/4)^{k-1} \cdot s \right) + \left(\sum_{i=1}^{k-2} (1/4)^i \cdot 2s + (1/4)^{k-1} \cdot s\right) + \left(\sum_{i=0}^{k-3} (1/4)^i \cdot 2s + (1/4)^{k-1} \cdot s \right) \notag \\
    = & 6s - 10s / 4^{k-1}.
\end{align}
We see that, for any hierarchy depth, Work Units is bounded from above by $6s$.

As discussed in Section \ref{sec:training_methods_for_nn}, deep learning, despite being formalized in the same way as a traditional optimization problem, has different goals from traditional optimization. Most significantly, the actual objective function of interest, the neural network's performance over the entire population of \emph{possible} data, is inaccessible, and the training performance is only an approximation to this. For this reason, the goal is not to achieve asymptotic convergence to a fixed minimum point, but rather to achieve good performance on a validation set.

\subsection{Poisson Problem}
We consider solving
\begin{equation}
-\nabla\cdot(\kappa\nabla u)=f \quad \mbox{in} \quad \Omega=[0,1]^2
\end{equation}
on unit square $\Omega=[0,1]^2$ with a $32\times 32$ mesh and $u=0$ on the boundary, where
\begin{align*}
 \kappa(x,y)&= 1.1 + \cos[k_x\pi(x'+a_x)]  \cos[k_y\pi(y'+a_y)]\\
 x'&=\cos\alpha(x-0.5)-\sin\alpha(y-0.5)+0.5\\
 y'&=\sin\alpha(x-0.5)+\cos\alpha(y-0.5)+0.5 \\
 f &=32e^{-4[(x-0.25)^2+(y-0.25)^2]}
\end{align*}
The data set has $10,000$ samples: frequencies $k_x,k_y \in (0.5,4)$, 
phase shifts $a_x,a_y\in (0, 0.5)$, and rotation angle $\alpha \in (0,\pi/2)$,
generated using FEniCS \cite{alnaes2015archive} and DOLFIN \cite{10.1145/1731022.1731030}.
The goal is to train NNs to learn the solution $u$ 
from $\kappa$, $f$, and the mesh point coordinates $(x,y)$.
So, the dimension of the input is $3\times32\times32$ and the dimension of the output is $32\times32$.
Figure \ref{fig:poisson} shows an example of an exact solution $u$ and the 
corresponding solution learned from NN.

\begin{figure}[htp]
    \centering
    \subfloat[$\kappa$]{
    \includegraphics[width=0.3\textwidth]{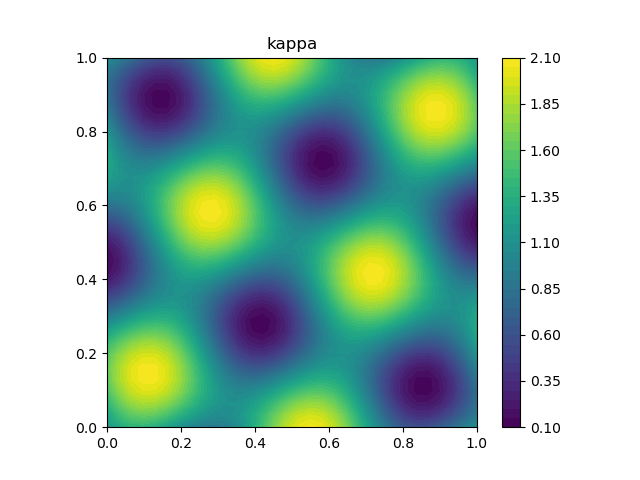}
    }
    \subfloat[exact $u$]{
    \includegraphics[width=0.3\textwidth]{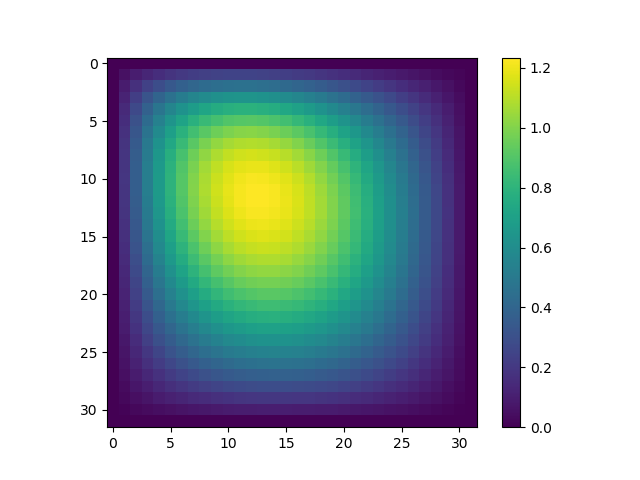}
    }
    \subfloat[learned $u$]{
    \includegraphics[width=0.3\textwidth]{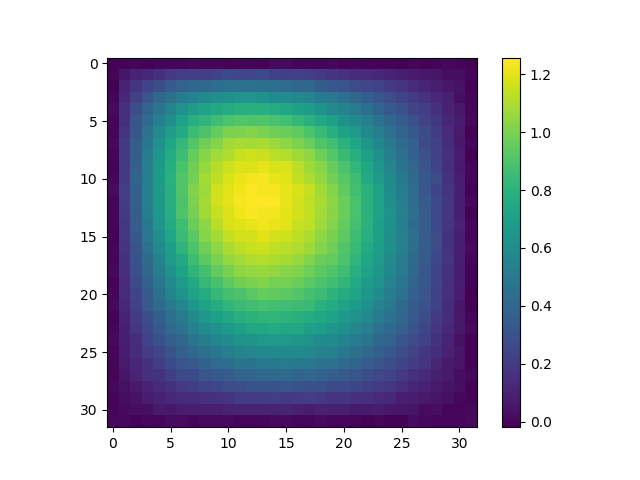}
    }
    \caption{Diffusion coefficient and the solution} 
    \label{fig:poisson}
\end{figure}


We tested two neural network architectures:
\begin{itemize}
    \item A fully-connected network with two hidden layers of 400 neurons each, with ReLU activation functions.
    \item A convolutional network with two convolutional layers of 200 channels each, kernel width 7 and stride 2, following by two fully-connected layers of 400 neurons each, with ReLU activation functions.
\end{itemize}

\begin{figure}
\centering
\subfloat[training $L_2$]{\includegraphics[width=.4\textwidth]{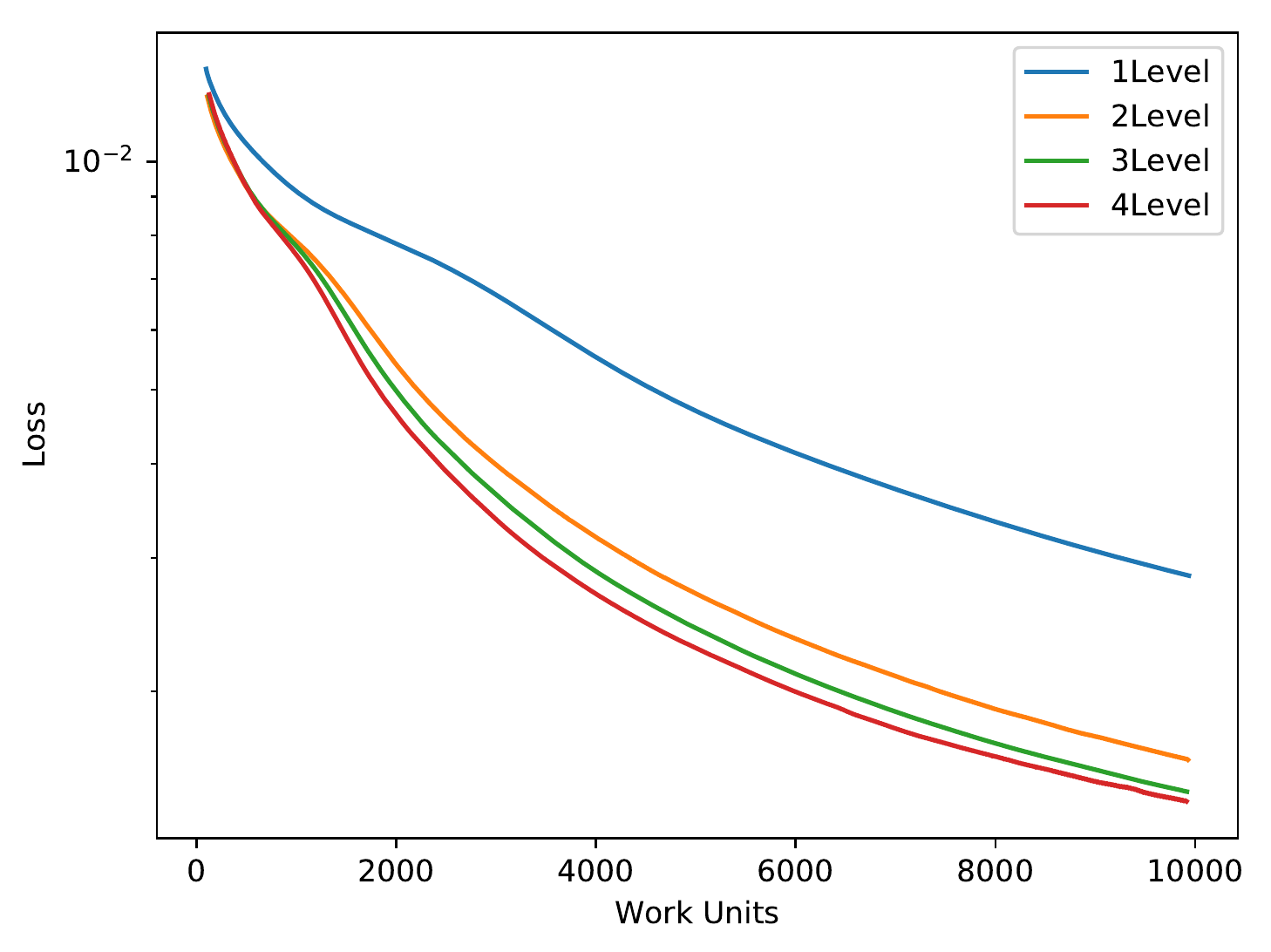}}%
\qquad
\subfloat[training $L_\infty$]{\includegraphics[width=.4\textwidth]{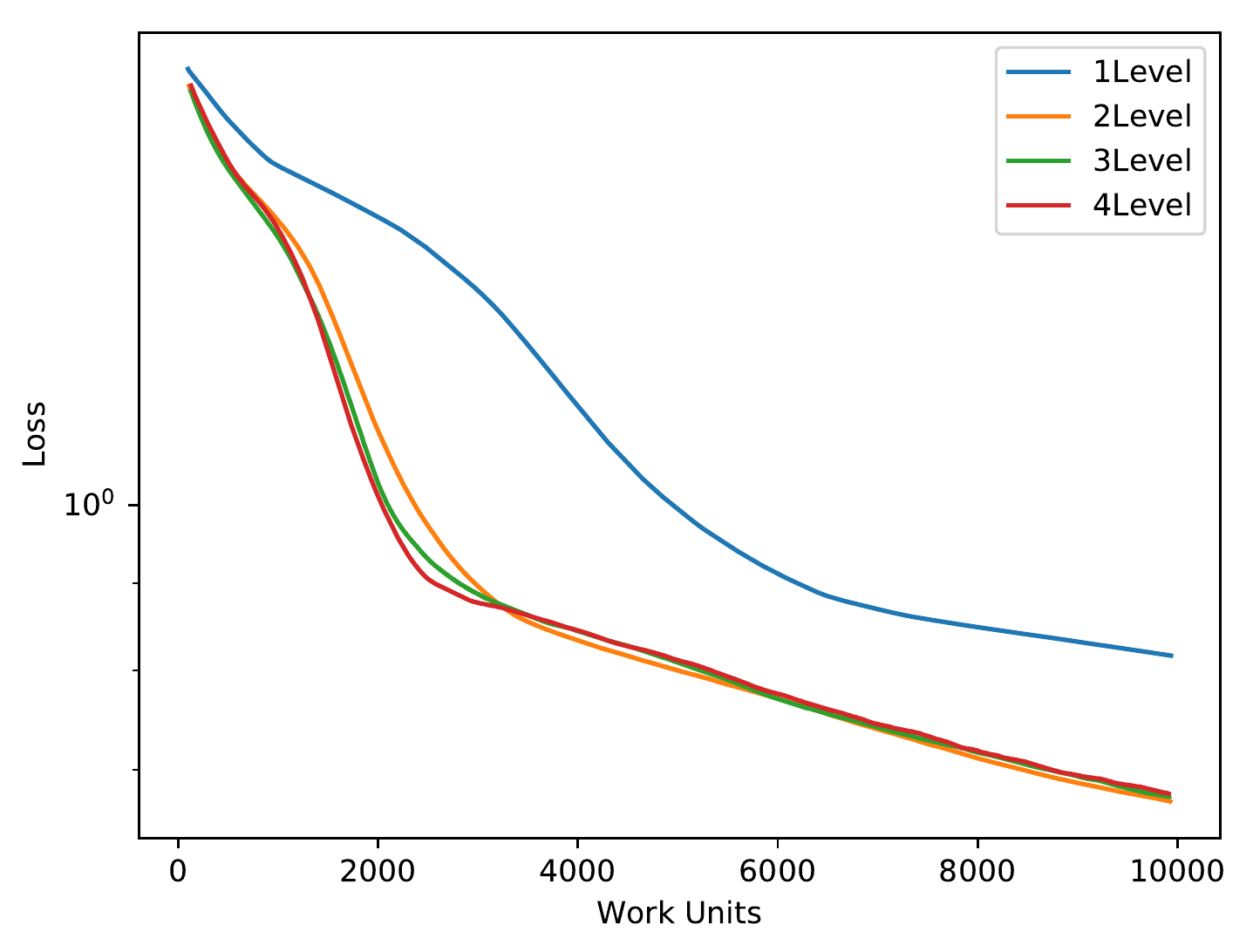}}%
\qquad
\subfloat[validation $L_2$]{\includegraphics[width=.4\textwidth]{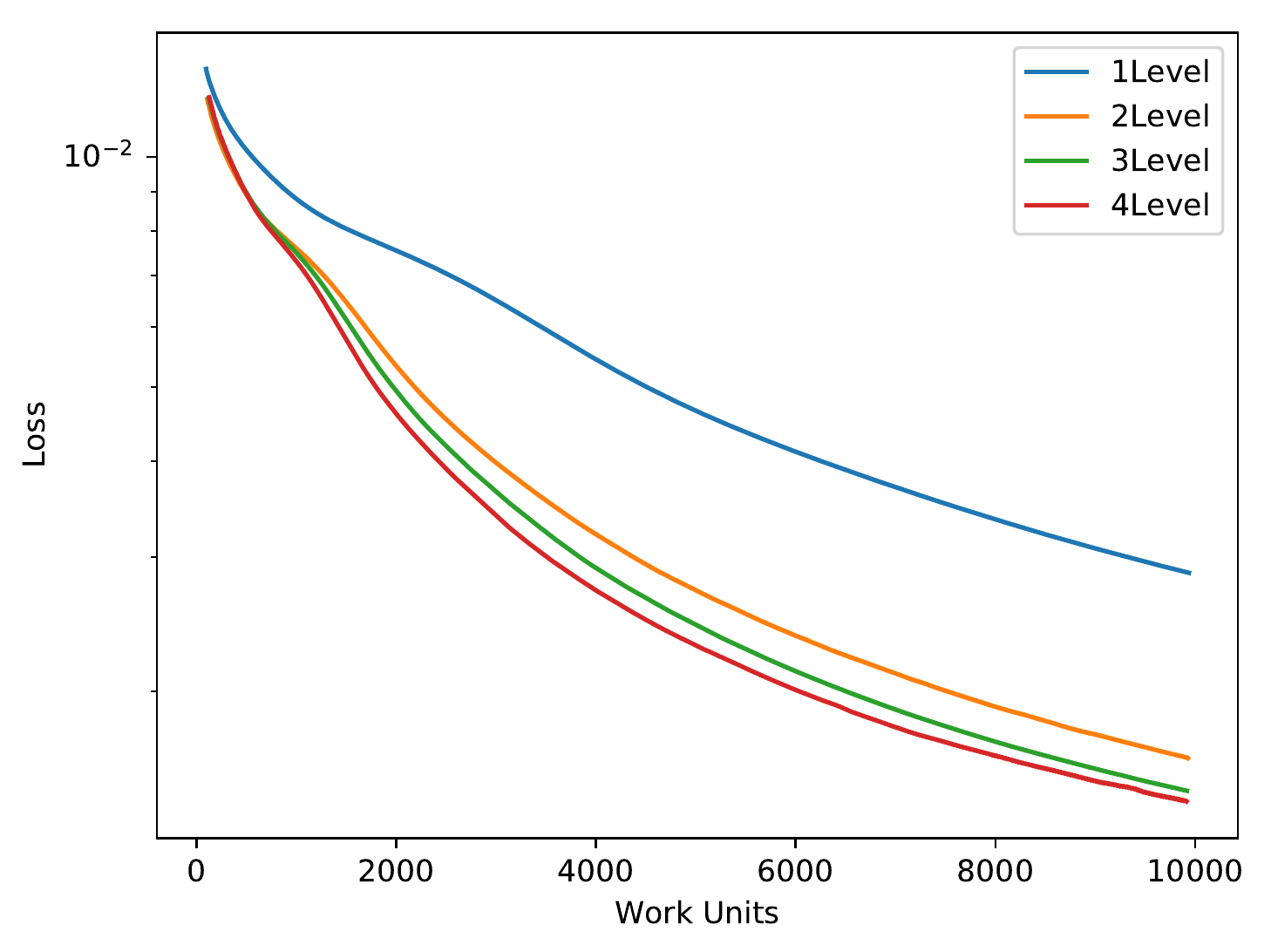}}%
\qquad
\subfloat[validation $L_\infty$]{\includegraphics[width=.42\textwidth]{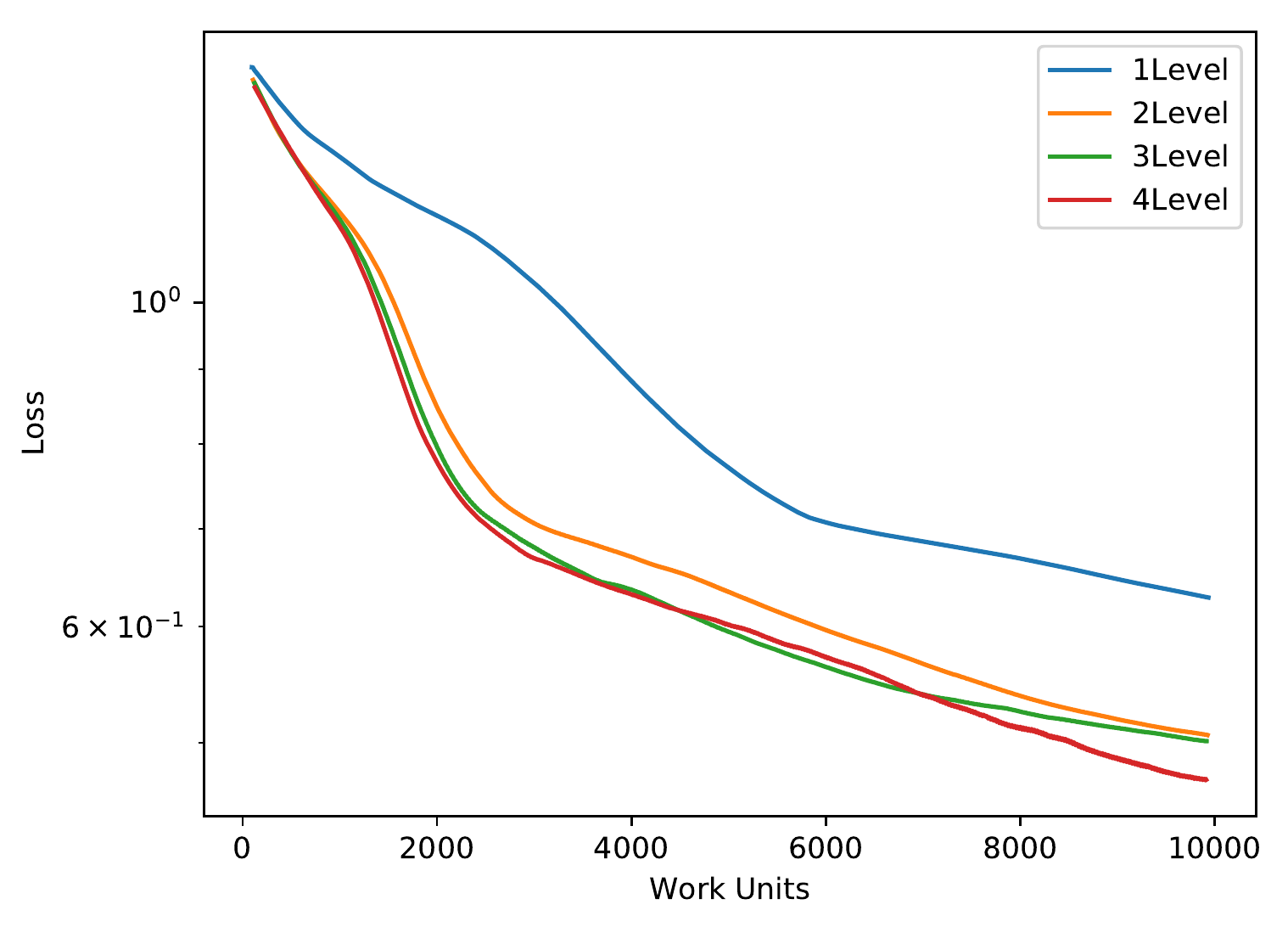}}%
\caption{(a) Smoothed training and validation loss ($L_2$ and $L_\infty$) for one-level SGD training   and multilevel FAS trainings (with 2-4 levels) for the Poisson problem using fully-connected NNs.}
\label{Poisson_results_fc}
\end{figure}

\begin{figure}
\centering
\subfloat[training $L_2$]{\includegraphics[width=.4\textwidth]{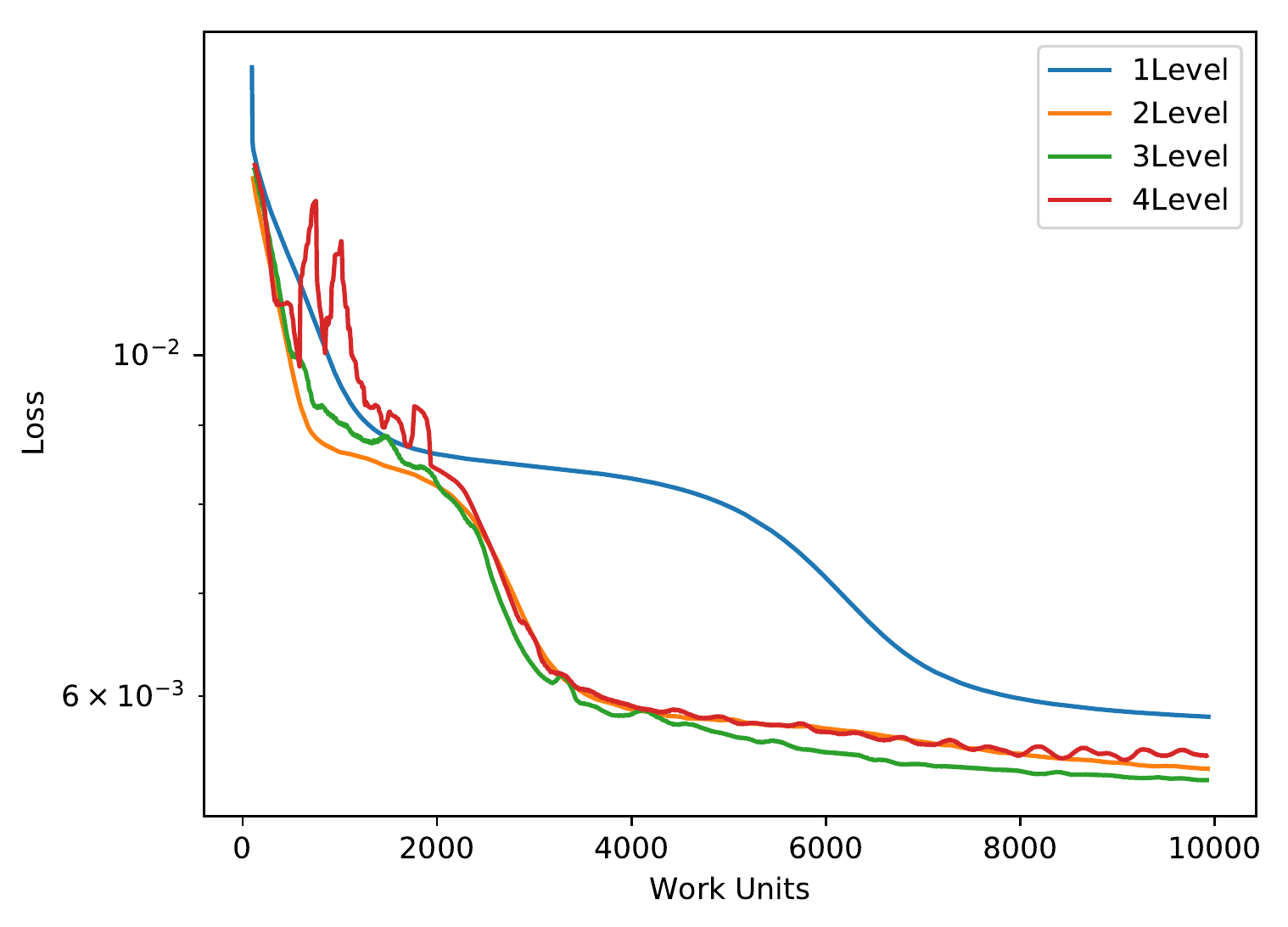}}%
\qquad
\subfloat[training $L_\infty$]{\includegraphics[width=.38\textwidth]{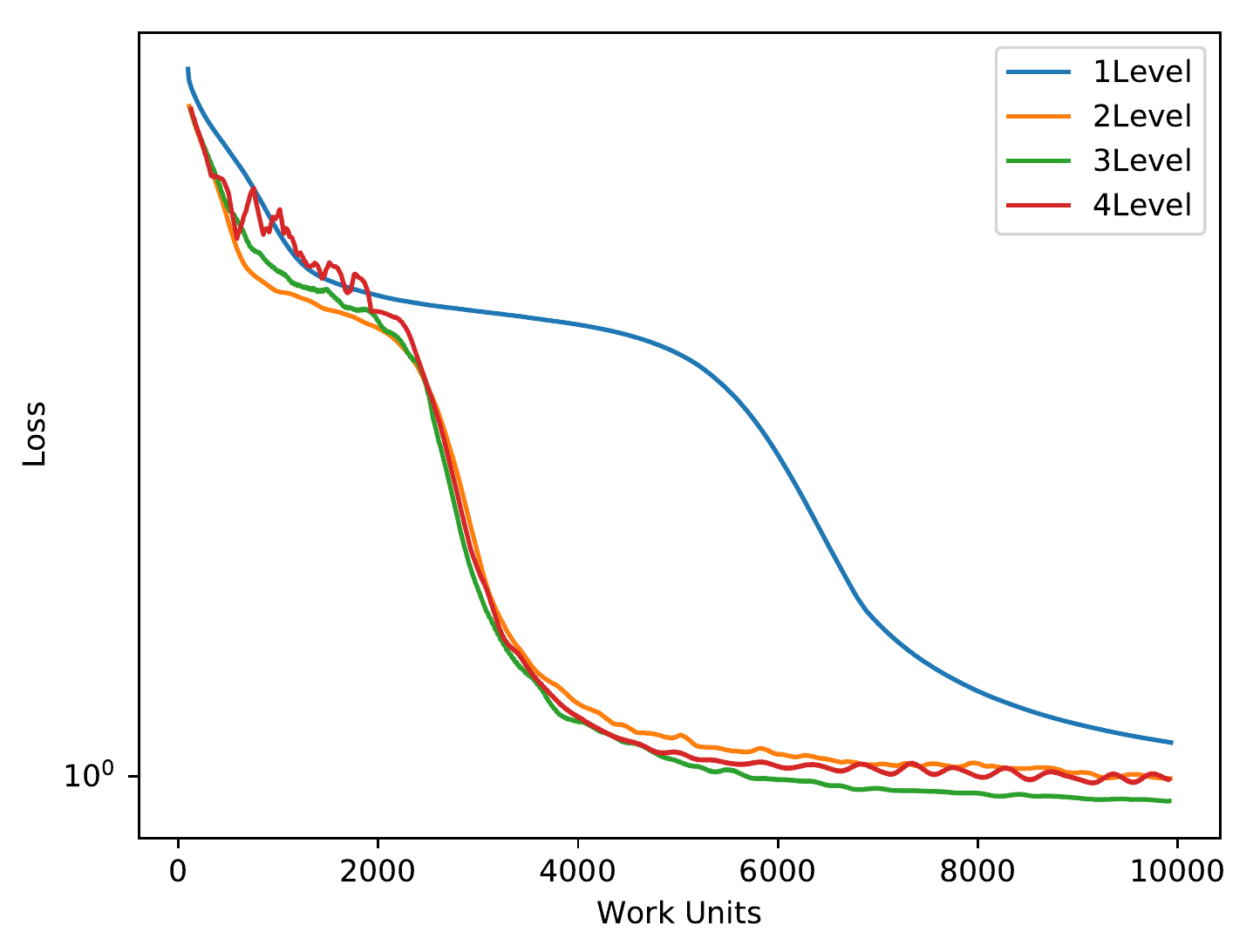}}%
\qquad
\subfloat[validation $L_2$]{\includegraphics[width=.4\textwidth]{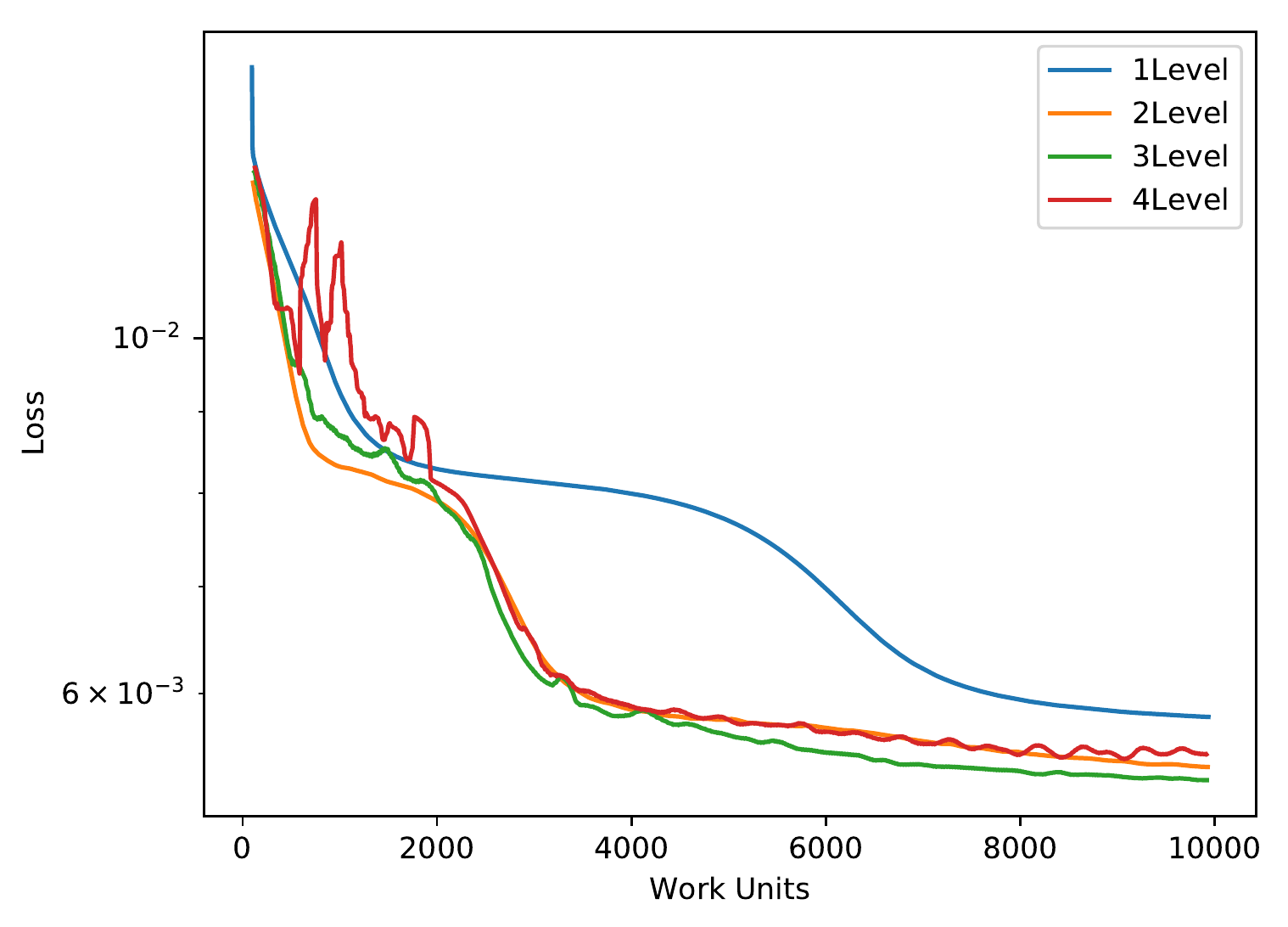}}%
\qquad
\subfloat[validation $L_\infty$]{\includegraphics[width=.4\textwidth]{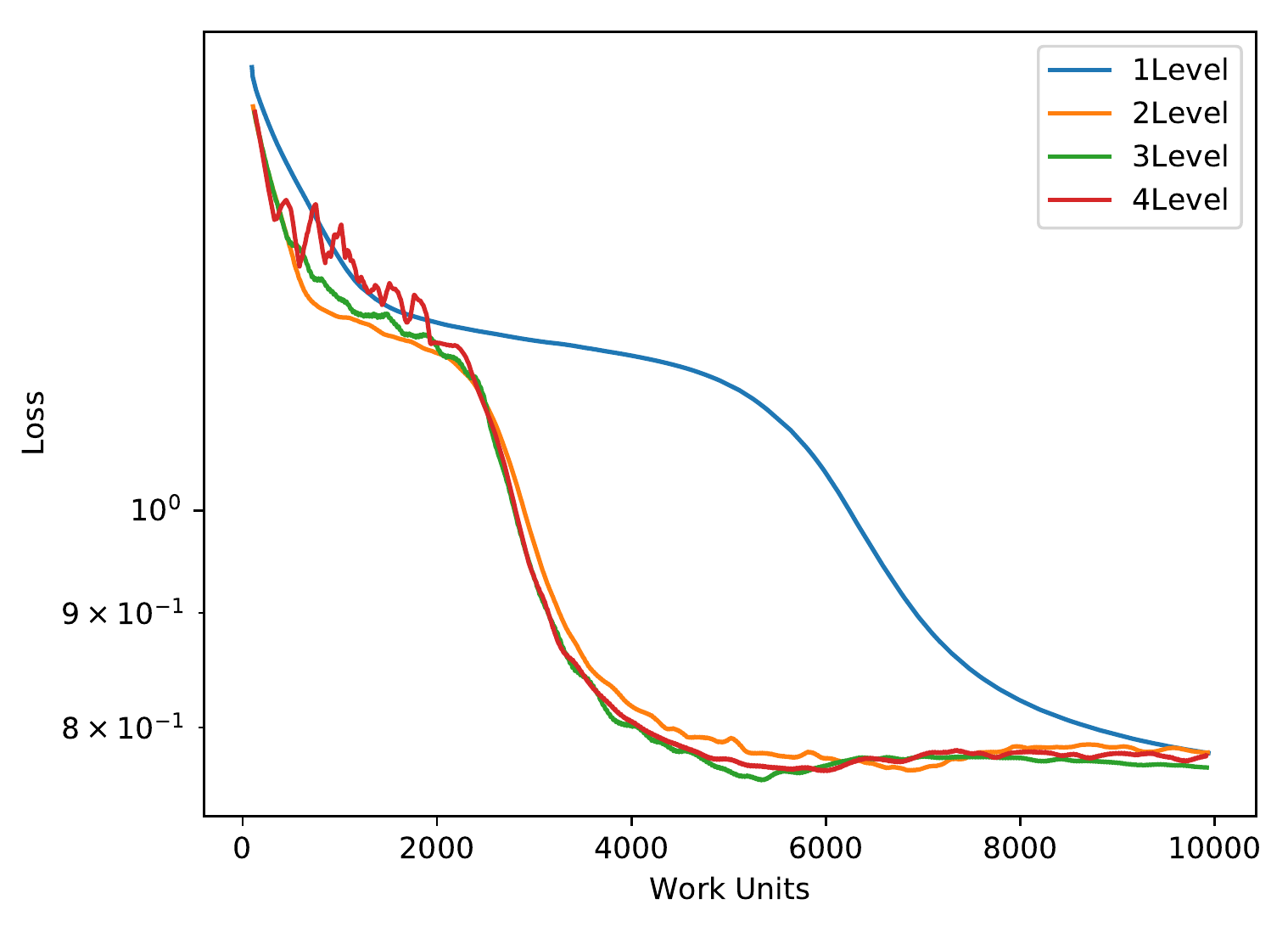}}%
\caption{(a) Smoothed training and validation loss ($L_2$ and $L_\infty$) for one-level SGD training   and multilevel FAS trainings (with 2-4 levels) for the Poisson problem using CNNs.}
\label{Poisson_results_cnn}
\end{figure}

We trained these neural networks using traditional (one-level) SGD and multilevel FAS with
SGD as the smoother (from 2 levels and up to 4 levels). In these cases, we used a learning rate of 0.01, momentum of 0.9, weight decay of 0, and minibatches of size 200.
For the multilevel FAS, we gradually decrease the learning rates for lower levels and
recompute the neuron matching for constructing the $P$ and $\Pi$ operators every 50 V-cycles.  We perform 4 SGD minibatch iterations at each smoothing step.
The results of validation and training losses are shown in Figure~\ref{Poisson_results_fc} for the fully-connected network and Figure~\ref{Poisson_results_cnn} for the convolutional network.
We also show the best losses achieved for each network with each method in Table~\ref{tab:best poisson losses} up to the first $10,000$ work units.

\begin{table}[h!]
\centering
\begin{tabular}{r|r|c|c|c|c|}
                                   &               & \multicolumn{2}{c|}{Validation Loss} & \multicolumn{2}{c|}{Training Loss} \\
                                   \hline
                                   &               & $L_2$       & $L_\infty$              &   $L_2$       & $L_\infty$          \\
                                   \hline
\multirow{2}{*}{Fully-Connected}  & One-Level SGD &  2.94e-3     &  0.626                  & {2.84e-3}       & 0.814              \\
                                  & Two-Level FAS &  \textbf{1.61e-3}     &  \textbf{0.496}                  & \textbf{1.59e-3}       & \textbf{0.664}              \\
                                  \hline
\multirow{2}{*}{Convolutional}   & One-Level SGD  & 5.79e-3      & 0.772                   & {5.81e-3}       & {1.022}              \\
                                 & Two-Level FAS  & \textbf{5.38e-3}      & \textbf{0.756}                   & \textbf{5.37e-3}       & \textbf{0.991}
\end{tabular}
\caption{Best $L_2$ and $L_\infty$ losses achieved for fully-connected and convolutional networks using both SGD and two-level FAS. We can see that the the fully-connected network outperforms convolutional networks, and the two-level FAS consistently outperforms SGD.}
    \label{tab:best poisson losses}
\end{table}

For the fully-connected NNs, we see that the multilevel FAS training methods have significantly faster convergence rates with respect to work units than the 1-level SGD in terms of both the $L_2$ and $L_\infty$ errors. 
With the 3- and 4-level methods, the speed of the convergence can be further improved slightly over the 2-level method.
Similar behavior can also be observed with CNNs, where multilvel FAS
methods yield superior convergence than SGD. We note here that  the 
4-level method appears to struggle with stability at the beginning of 
the training process. So, in this experiment,
we did not further increase the number of levels. 
Improving the stability of the proposed approach with higher number of
levels is left for our future work.

\subsection{Darcy Flow Problem} 

\begin{figure}
  \centering
  \includegraphics[width=6cm]{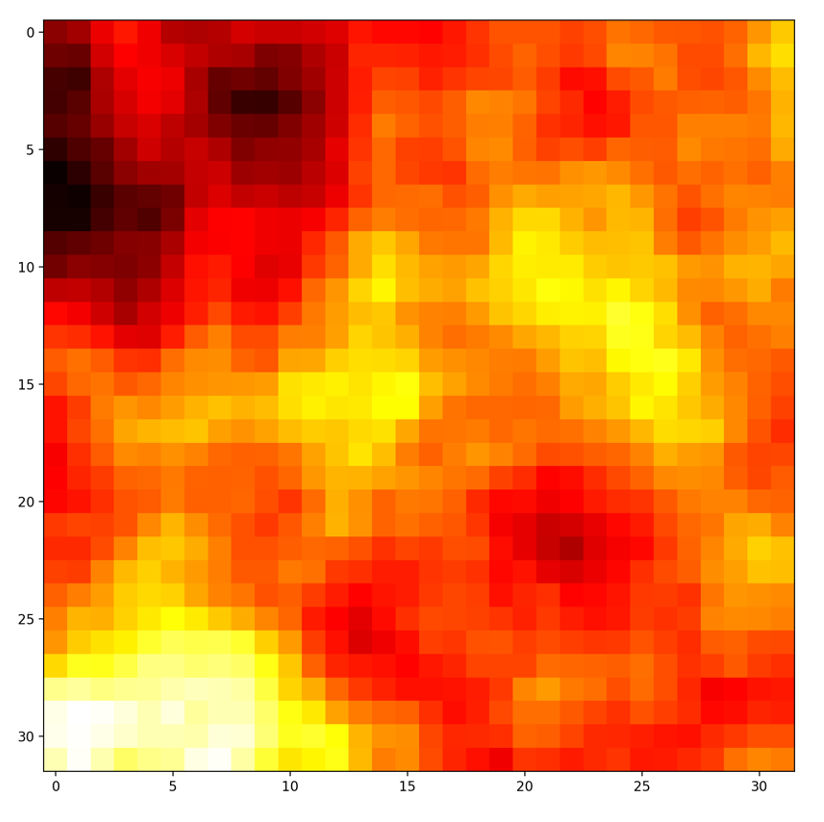}
  \caption{A spatially-correlated permeability field used as input to a Darcy problem.}
  \label{darcy}
\end{figure}

The test data correspond to the finite element solution of the Darcy equation. More specifically, for a single phase porous media flow, we compute the flux $\bq$ and the pressure $p$ by solving 
the following mixed system posed on the unit square $D$:  Given two  (Raviart-Thomas) finite element spaces, $\bV_h \subset \bH(\div,\;D)$ and $W_h \subset L_2(D)$,
find $\bq \in \bV_h$ and $p \in W_h$, such that 
\begin{equation}\label{mixed Darcy system}
\begin{array}{lll}
(k^{-1} \bq,\;\bs) - (\div \bs,\;p) & = (\bbf,\;\bs), &\text{ for all }\bs \in {\bV}_h,\\
(\div \bq,\;w) & = 0,& \text{ for all } w \in W_h,
\end{array}
\end{equation}
subject to Dirichlet boundary conditions $p = p_D$ on $\Gamma_D$ enforced by the right-hand side $\bbf$, and Neumann boundary condition $\bq \cdot \bn = 0$ on $\Gamma_N$, where $\Gamma_D $ and $\Gamma_N$ are nonintersecting and cover the boundary  $\partial D$. 

For a given mesh $\T_h$, we have the following data:
\begin{itemize} 
\item 
The input is $k = \exp(u)$, where $u$ is a realization of log-normal permeability field with piecewise constant values over the elements of $\T_h$.
\item 
The output data is either the solution $p \in W_h$, also a piecewise constant quantity over the elements of $\T_h$, or a specific quantity of interest, QoI.
In our case, we choose as QoI, 
\begin{equation}\label{QoI}
Q = \frac{1}{|\Gamma_{out}|}\; \int\limits_{\Gamma_{out}} \bq \cdot \bn\; dS,
\end{equation}
representing the average flux across the outflow boundary $\Gamma_{out}$, where $\bn$ is outer unit vector normal to $\Gamma_{out} \subset \partial D$ \cite{fairbanks}.
\end{itemize}
This type of data $(u,\;Q)$ is used in Monte Carlo simulations to estimate the expected value of the quantity of interest. See Figure~\ref{darcy} for an example of a permeability field $u$.

\begin{figure}[h]
\centering
\subfloat{\includegraphics[width=.4\textwidth]{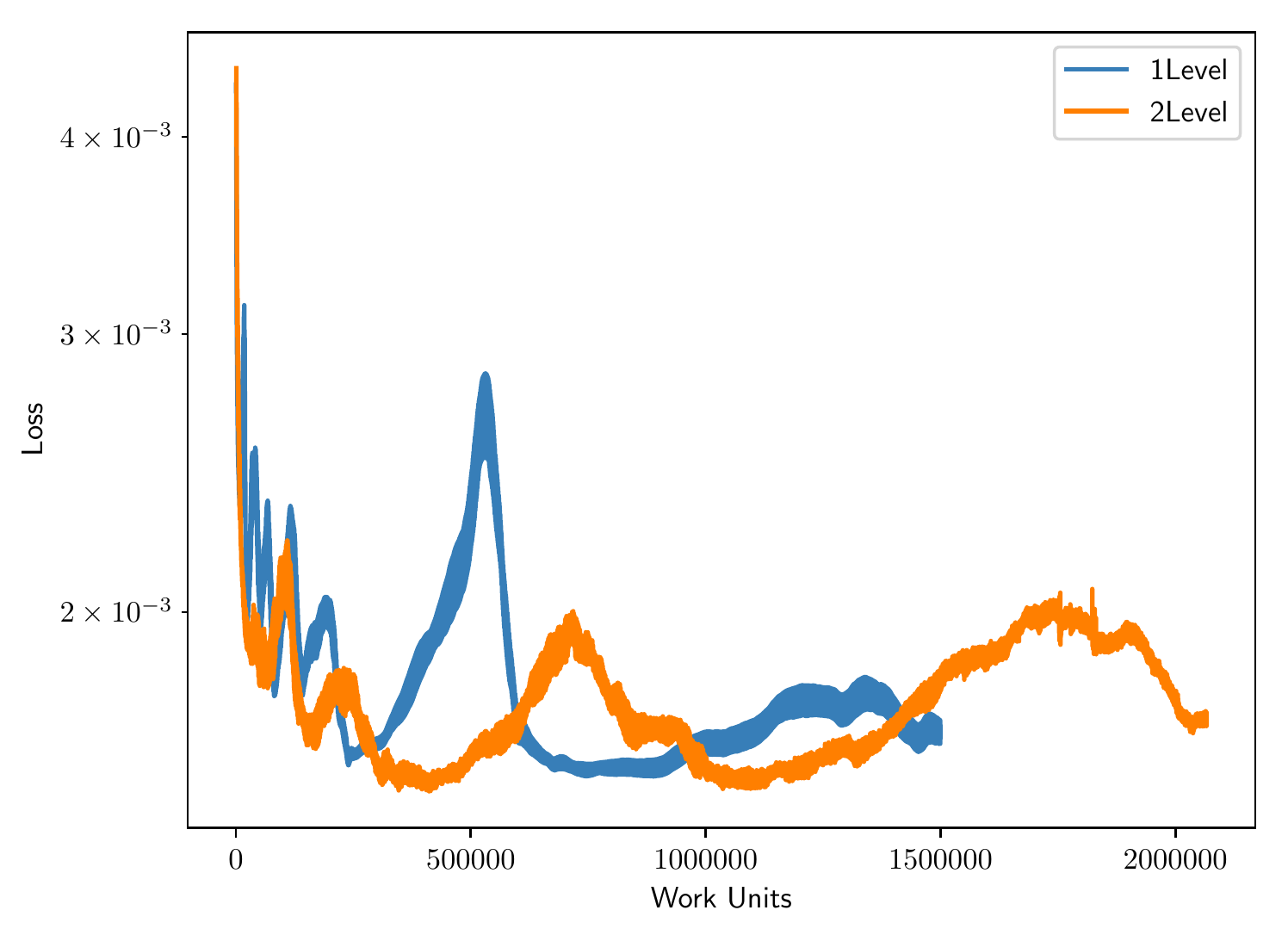}}%
\qquad
\subfloat{\includegraphics[width=.4\textwidth]{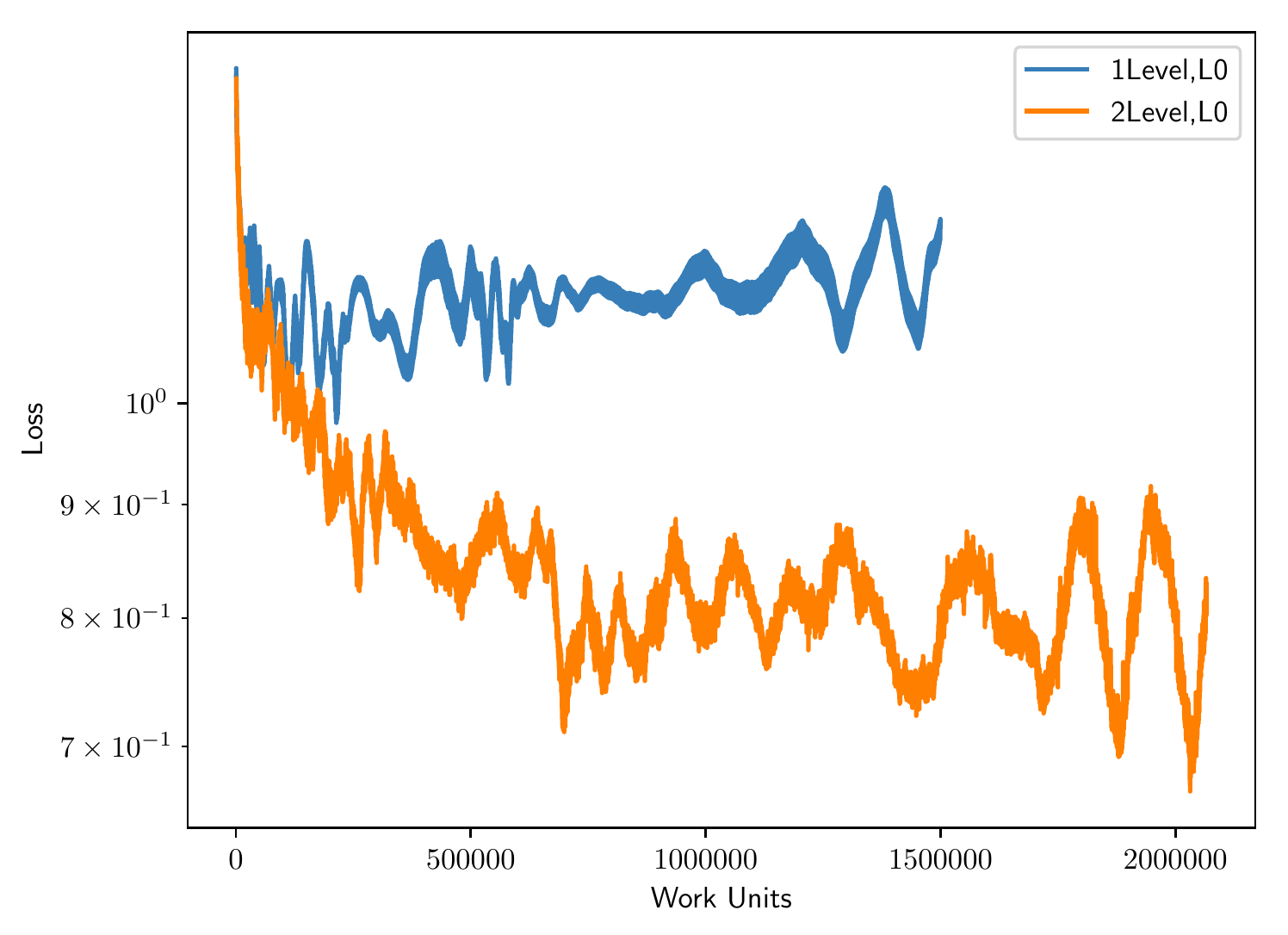}}%
\caption{(a) Smoothed $L_2$ loss and (b) smoothed $L_\infty$ loss for one-level training (i.e. SGD) and two-level FAS for the Darcy problem using a fully-connected network.}
\label{darcy_results_fc}
\end{figure}

\subsubsection{2-Level Without Tuning}

\begin{figure}[h]
\centering
\subfloat{\includegraphics[width=.4\textwidth]{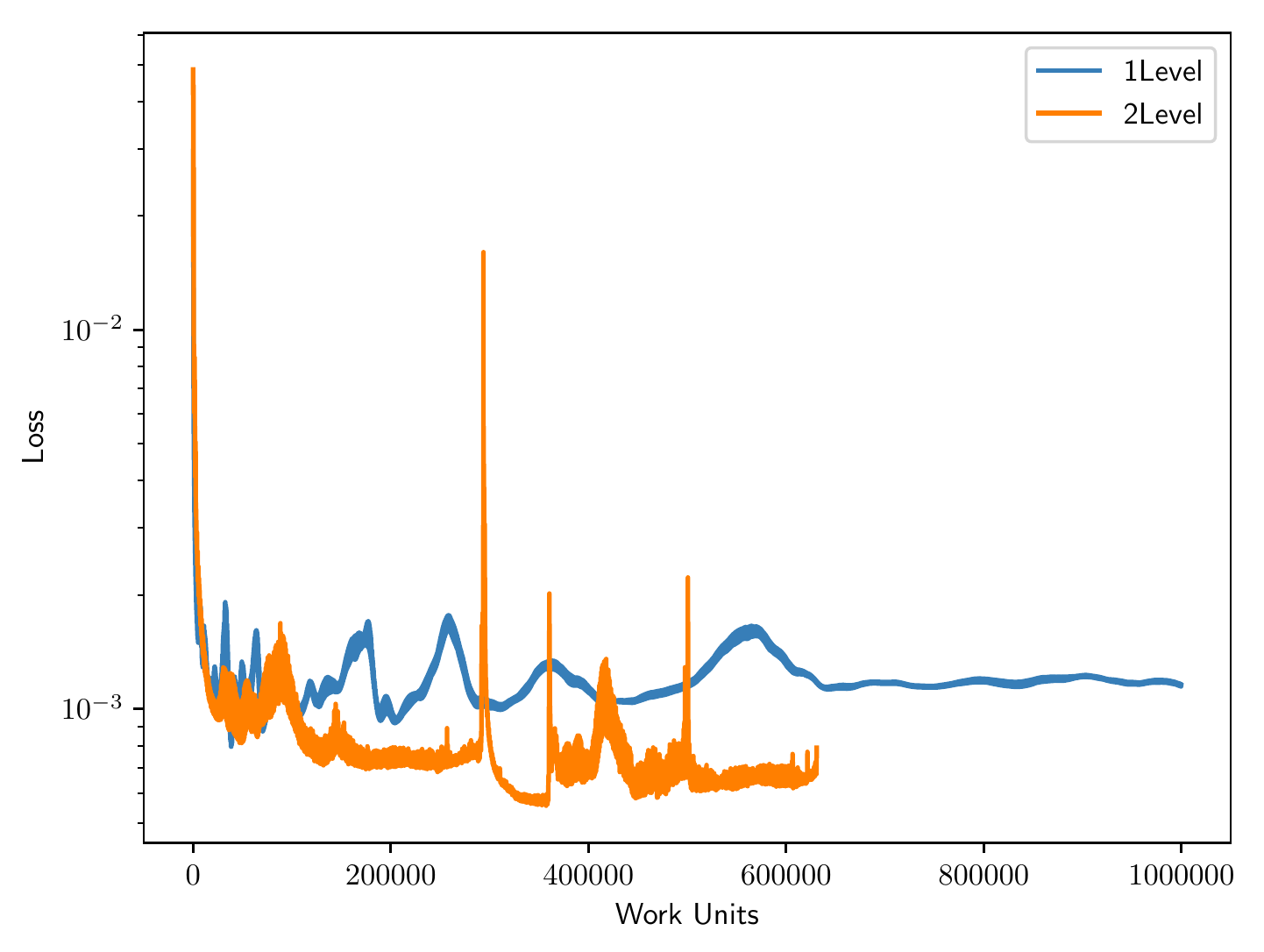}}%
\qquad
\subfloat{\includegraphics[width=.4\textwidth]{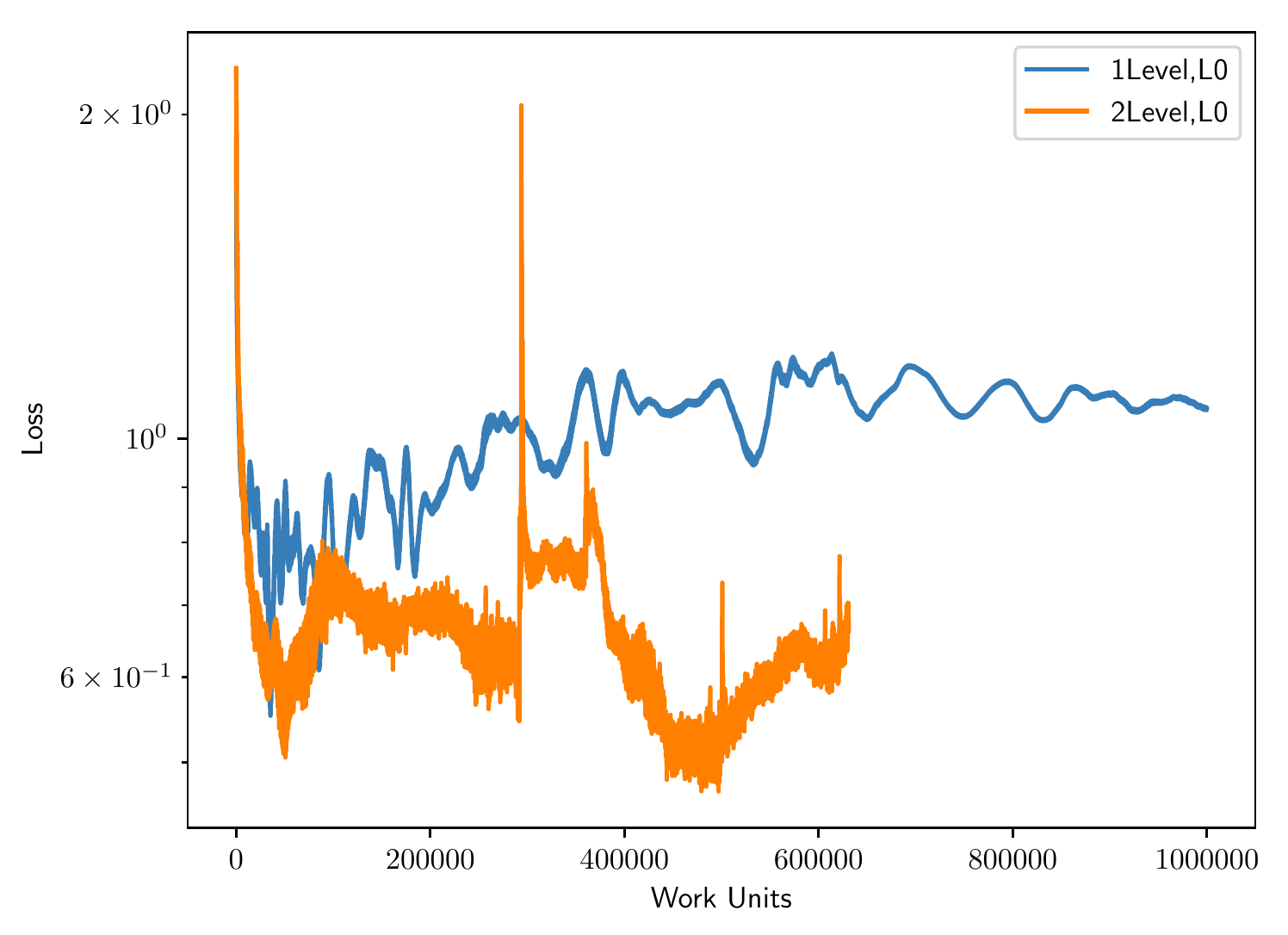}}%
\caption{(a) Smoothed $L_2$ loss and (b) smoothed $L_\infty$ loss for one-level training (i.e. SGD) and two-level FAS for the Darcy problem using a CNN.}
\label{darcy_results_cnn}
\end{figure}

We show the validation losses achieved for two types of networks:
\begin{itemize}
    \item A fully-connected network with 3 layers, each with 1024 neurons.
    \item A convolutional neural network with 3 convolutional layers. The channels, kernel width, and stride of these layers are (100, 200, 400), (11, 7, 3), and (2, 1, 1) respectively. Following these are two fully-connected layers each with 2048 neurons.
\end{itemize}
We trained these two networks using both traditional (one-level) SGD and our two-level FAS. In these cases we used a learning rate of 0.01, momentum of 0.9, weight decay of 1e-6, with minibatches of size 200. For two-level FAS, we recompute matchings with which to construct the $P$ and $\Pi$ operators every 10 V-cycles, and perform 2 SGD minibatch iterations at each smoothing step.

\begin{table}[h!]
\centering
\begin{tabular}{r|r|c|c|c|c|}
                                   &               & \multicolumn{2}{|c|}{Validation Loss} & \multicolumn{2}{|c|}{Training Loss} \\
                                   \hline
                                   &               & $L_2$       & $L_\infty$              &   $L_2$       & $L_\infty$          \\
                                   \hline
\multirow{2}{*}{Fully-Connected}  & One-Level SGD &  1.49e-3     &  0.707                  & \textbf{2.37e-5}       & 0.0546              \\
                                  & Two-Level FAS &  \textbf{1.39e-3}     &  \textbf{0.472}                  & 4.75e-5       & \textbf{0.0540}              \\
                                  \hline
\multirow{2}{*}{Convolutional}   & One-Level SGD  & 6.61e-4      & 0.361                   & \textbf{2.92e-5}       & \textbf{0.0469}              \\
                                 & Two-Level FAS  & \textbf{4.84e-4}      & \textbf{0.312}                   & 7.47e-5       & 0.0696
\end{tabular}
\caption{Best $L_2$ and $L_\infty$ losses achieved for fully-connected and convolutional networks using both SGD and two-level FAS. We can see that the CNN consistenly outperforms the fully-connected network, and the two-level FAS consistenly outperforms SGD.}
    \label{tab:best darcy losses}
\end{table}

The results can be seen in Figure~\ref{darcy_results_fc} for the fully-connected network and Figure~\ref{darcy_results_cnn} for the convolutional network using both traditional SGD and our two-level FAS. As stochastic training methods tend to produce high performance variation from step to step, raw plots of performance over many steps can be messy; to improve the visualization these plots are smoothed using a moving arithmetic mean with a window size of 33, chosen as the smallest window size we felt produced informative visuals.  We also show the best losses achieved for each network with each method in Table~\ref{tab:best darcy losses}. We can see from these results that, as with the Poisson problem, two-level FAS produces significant improvements over SGD.

\subsection{Deeper Hierarchies with Stability Tuning}

Here we investigate the use of deeper hierarchies. We find that the 3-level hierarchy tends to suffer from stability problems, while the 4-level and 5-level hierarchies are much more stable. Understanding this stability phenomena is the subject of future work.

In order to maintain stability, we make the following hyperparameter choices:
\begin{itemize}
    \item The momentum coarse-grid correction scaling parameter is set to $\alpha_m = 0.2$.
    \item The tau correction scaling parameter is set to $\gamma = 0.125$.
    \item The learning rate on each hierarchy level is divided by a factor of $\eta = 2 \sqrt{2}$ for the first 3 levels, after which the learning rate is kept the same as the third level.
\end{itemize}

\begin{figure}
\centering
\begin{tabular}{cc}
\subfloat{\includegraphics[width=.37\textwidth]{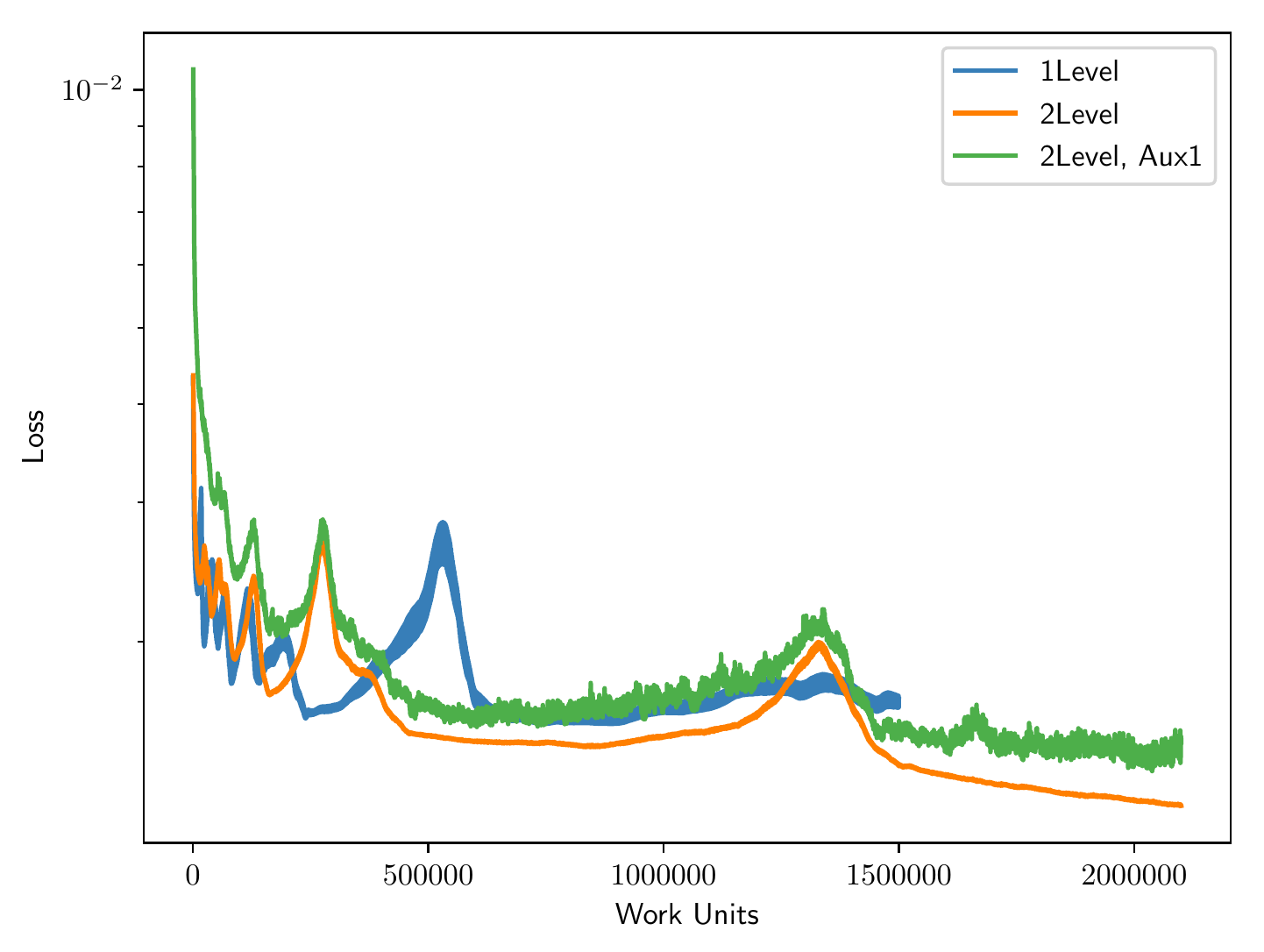}}%
&
\subfloat{\includegraphics[width=.37\textwidth]{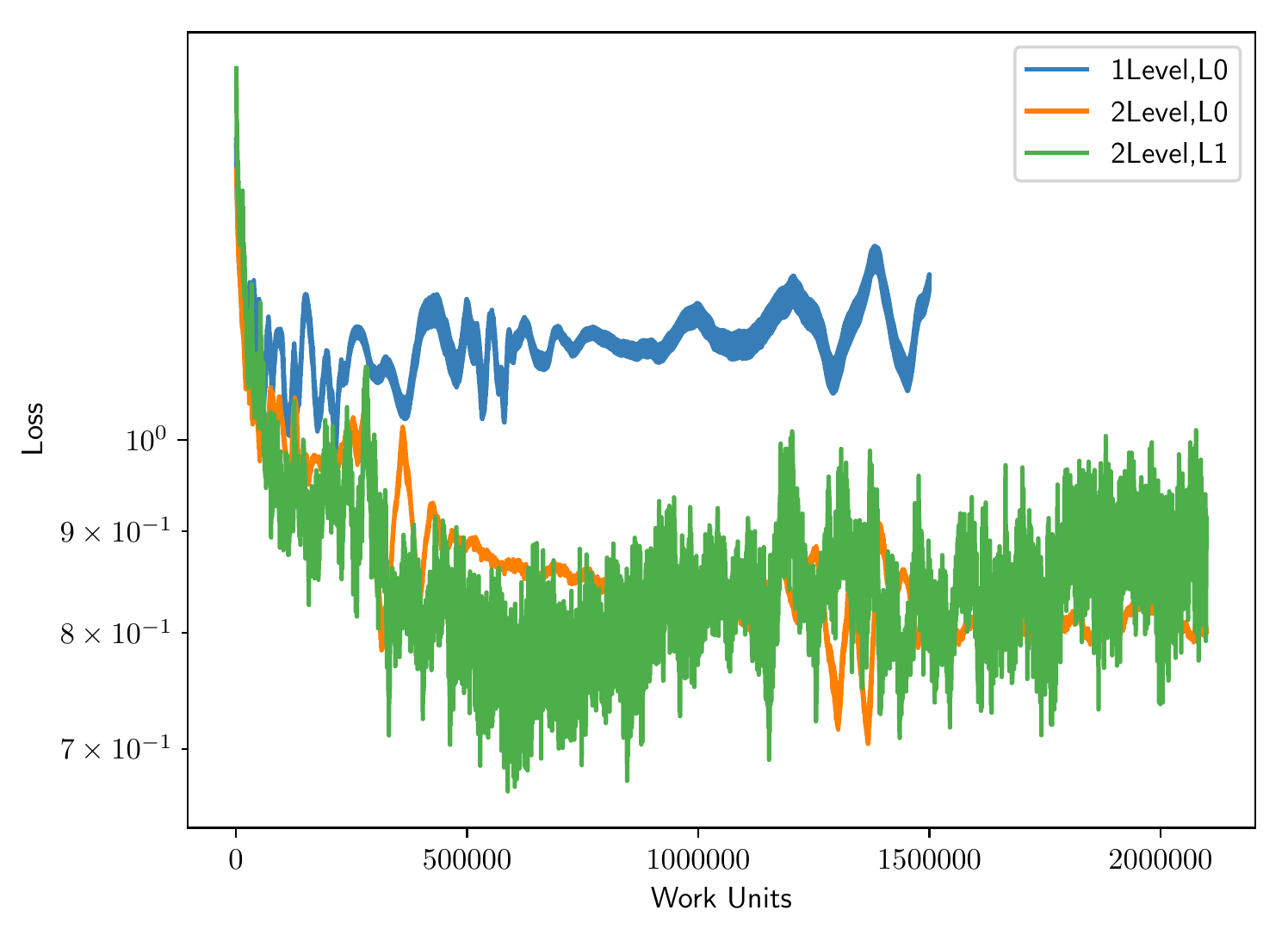}}%
\\
\subfloat{\includegraphics[width=.37\textwidth]{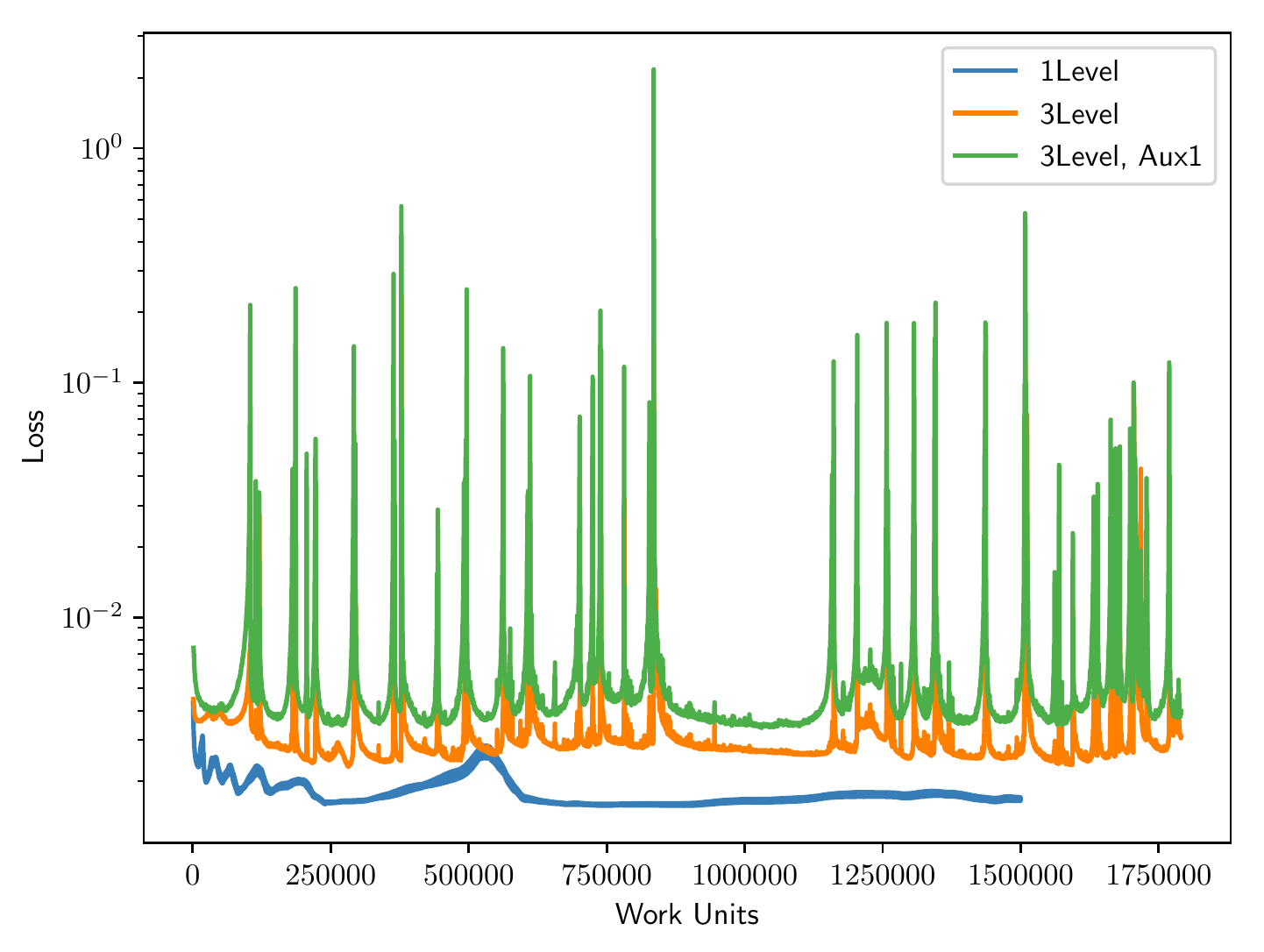}}%
&
\subfloat{\includegraphics[width=.37\textwidth]{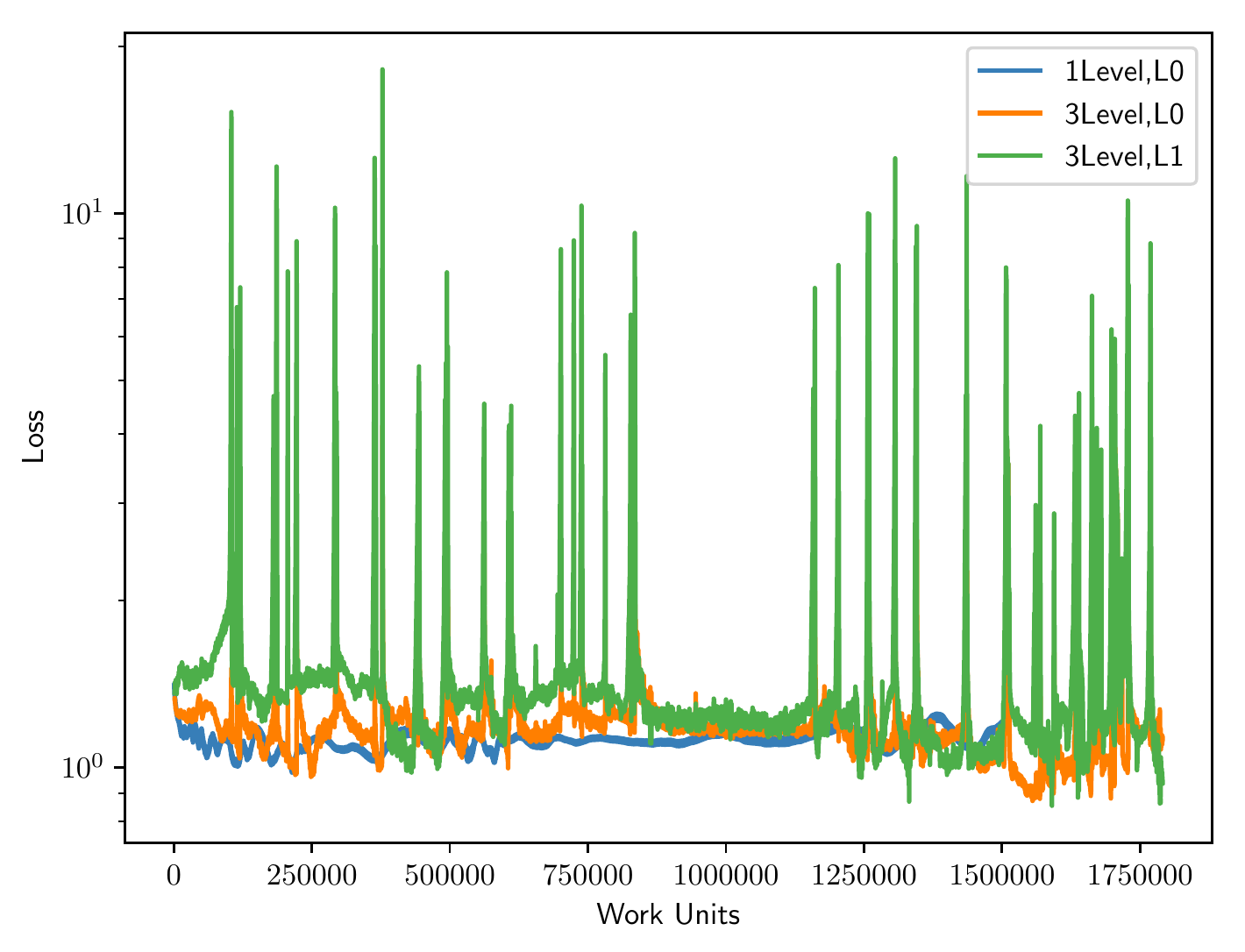}}%
\\
\subfloat{\includegraphics[width=.37\textwidth]{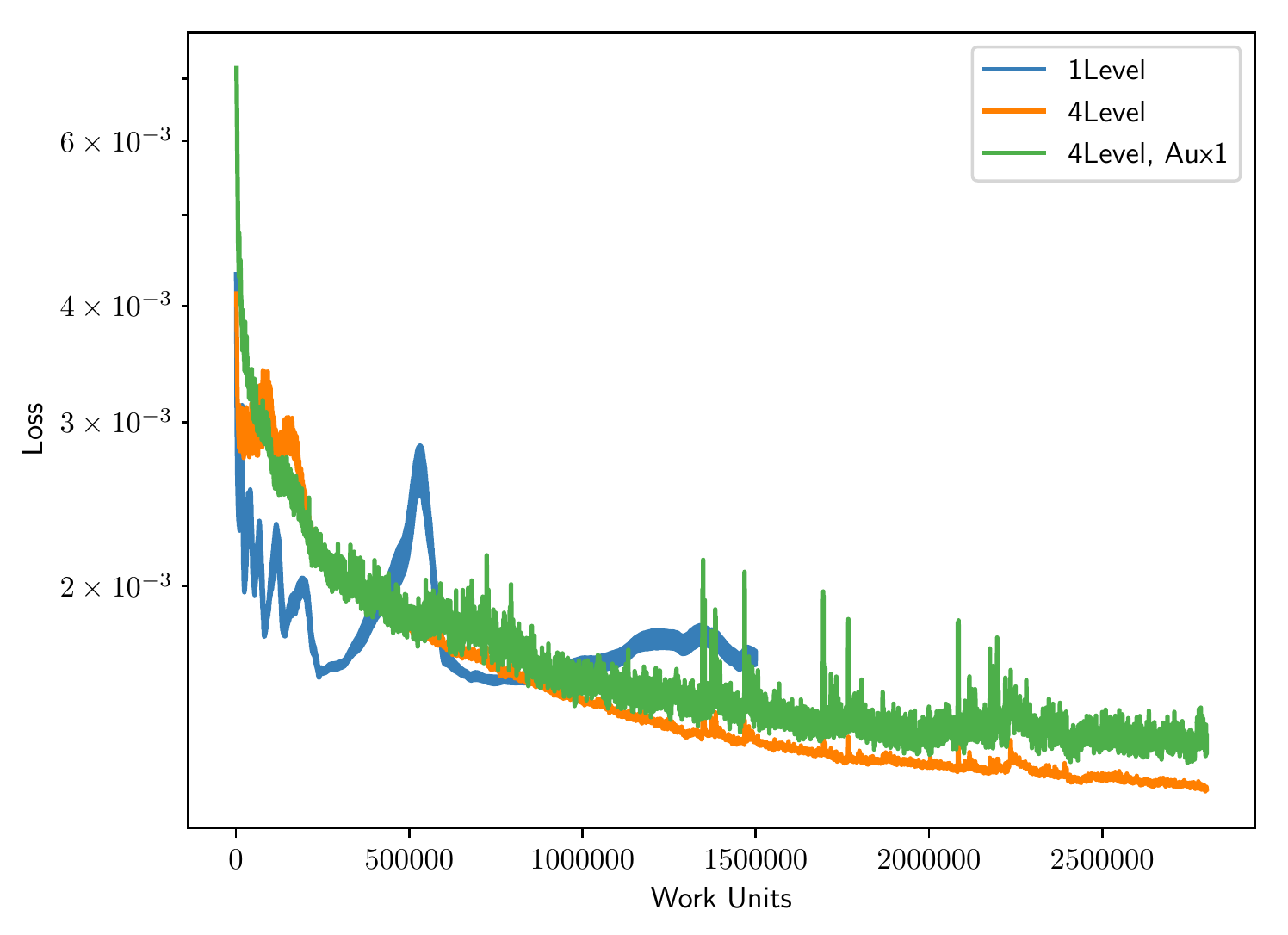}}%
&
\subfloat{\includegraphics[width=.37\textwidth]{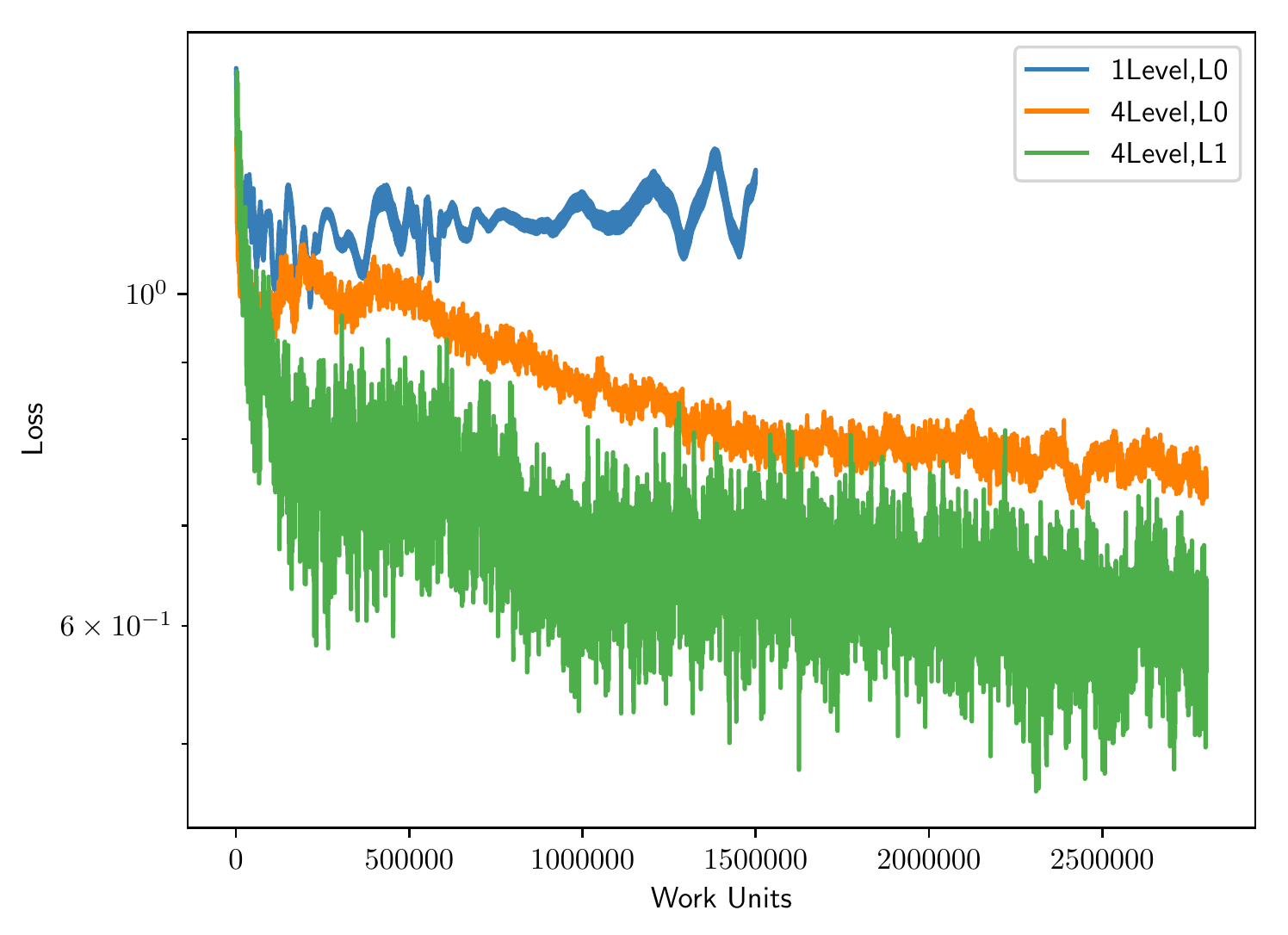}}%
\\
\subfloat{\includegraphics[width=.37\textwidth]{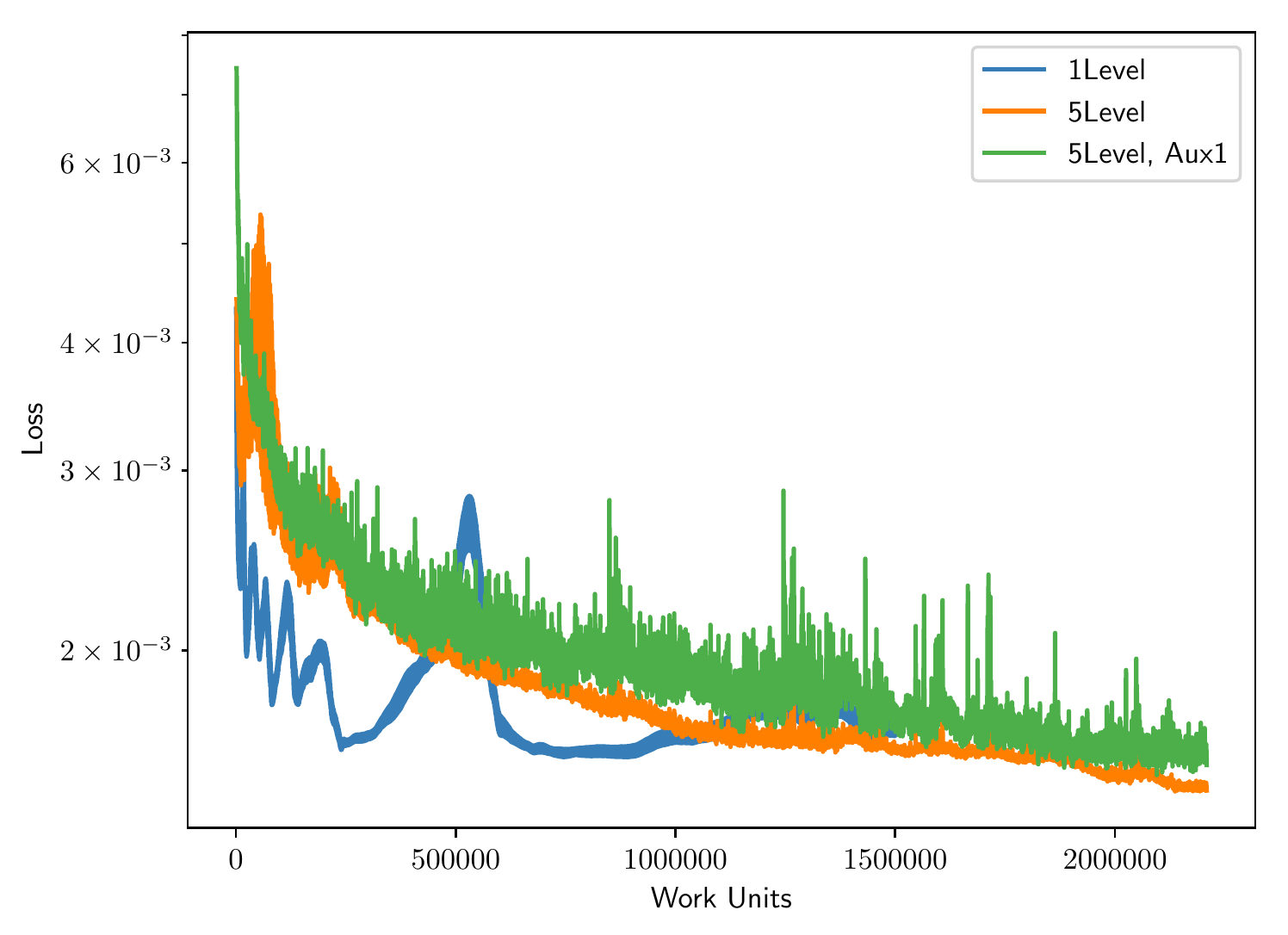}}%
&
\subfloat{\includegraphics[width=.37\textwidth]{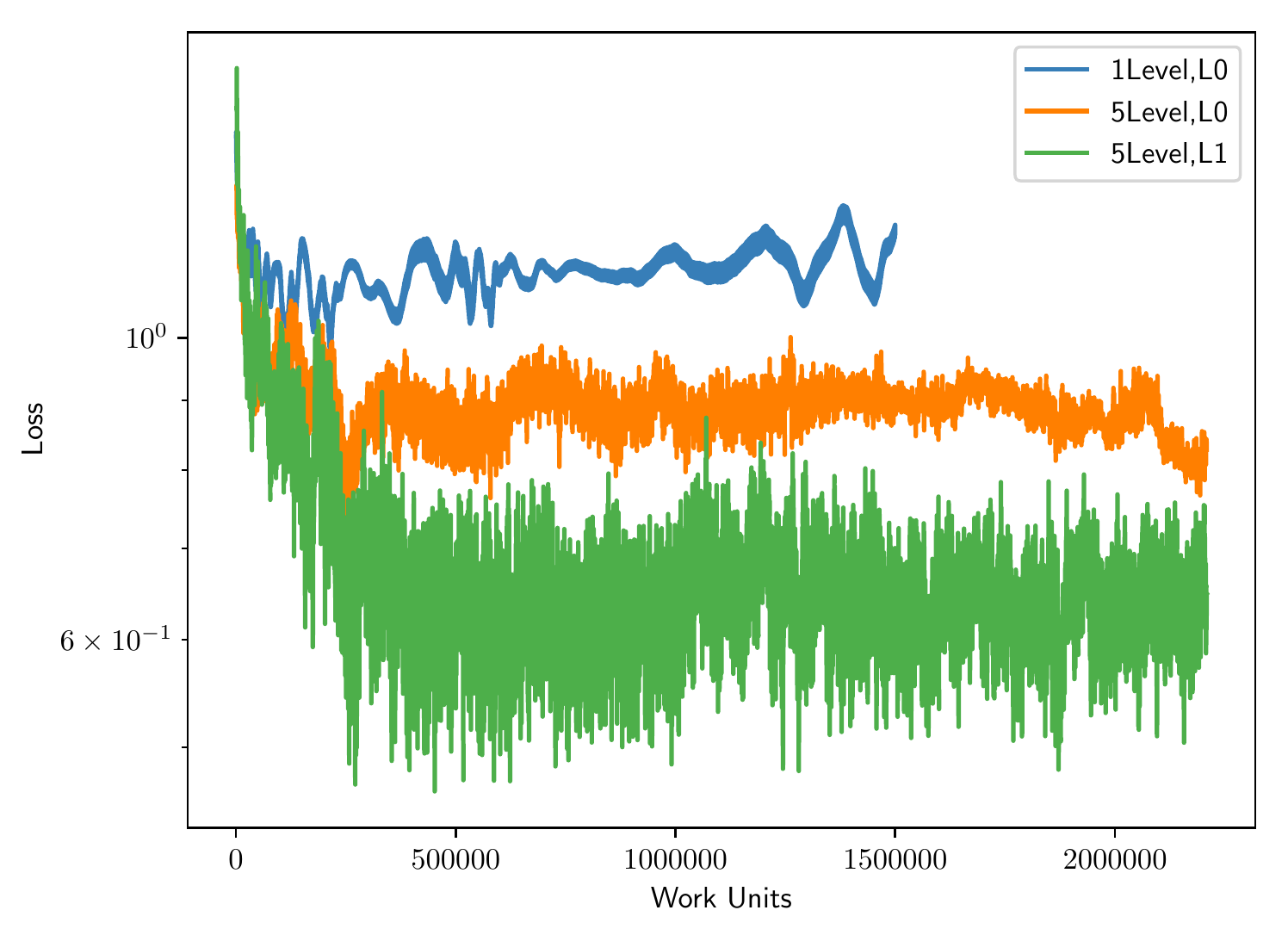}}%
\end{tabular}
\caption{$L_2$ loss (left) and $L_\infty$ validation loss (right) for fully-connected networks trained with 2, 3, 4, and 5 hierarchy levels. Each plot shows 1-level loss (the same in each plot), multilevel loss, and first auxiliary loss.}
\label{darcy_results_tuned}
\end{figure}

The results can be seen in Figure \ref{darcy_results_tuned}, where we show validation error from both the original neural network and its first-coarsened neural network. In each plot in this Figure, we also show the 1-level SGD results as a point of common comparison. We find that not only does the fine neural network perform well, but that the first auxiliary network (the first-coarsened network) often has strong performance as well. In fact, what one can observe in Figure \ref{darcy_results_tuned} is that the original neural network typically performs best in the L2 error, while the first auxiliary network performs best in the $L_\infty$ error. We believe this is because the first-coarsened network has roughly half as many neurons and therefore has less opportunity for overfitting. We have found that this strong auxiliary performance only occurs when the tau scaling parameter $\gamma$ is sufficiently small. We also find that the 3-level hierarchy appears unusually unstable; this is not fully understood, and deserves further investigation.

We also show the best losses seen, for both validation and training loss across different hierarchy depths with both fully-connected and convolutional networks, in Table \ref{tab:deep darcy losses}. These results demonstrate the regularization impact of multilevel training. When dealing with training loss, either the one-level method or a shallower hierarchy produces the best losses. But when considering validation loss, which is the stronger metric of generalization performance of the network, the deeper hierarchies produce significantly stronger performance. This is especially true when considering $L_\infty$ loss; this is consistent with the notion of regularization as a mechanism to prevent inappropriate corner-case behavior.

\begin{table}
\centering
\begin{tabular}{r|l|l|l|l||l|l|l|l|}
                                                  & \multicolumn{4}{|c|}{Full-Connected Network}                                &   \multicolumn{4}{|c|}{Convolutional Network}                                \\
        & \multicolumn{2}{|c|}{Validation Loss}& \multicolumn{2}{|c|}{Training Loss} &  \multicolumn{2}{|c|}{Validation Loss} & \multicolumn{2}{|c|}{Training Loss} \\
\hline
Levels  & $L_2$            & $L_\infty$       &   $L_2$          & $L_\infty$      &   $L_2$       & $L_\infty$        &   $L_2$       & $L_\infty$           \\
\hline
1       &  1.49e-3         &  0.677           & 2.20e-5          & 0.0544          & 6.61e-4          & 0.361          & \textbf{2.92e-5} & \textbf{0.0469}        \\
2       &  1.21e-3         &   0.434          & \textbf{4.66e-6} & \textbf{0.0165} & 6.06e-4          & 0.371          & 5.38e-5          & 0.0867       \\
2aux    &  1.35e-3         &   0.492          & 4.25e-4          & 0.107           & 6.32e-4          & 0.392          & 2.53e-4          & 0.0954        \\
3       &  2.15e-3         &   0.573          & 1.31e-3          & 0.747           & 1.05e-3          & 0.361          & 6.84e-4          & 0.686       \\
3aux    &  3.16e-3         &   0.642          & 2.28e-3          & 0.656           & 1.34e-3          & 0.512          & 1.03e-3          & 0.395       \\
4       & \textbf{1.20e-3} &   0.484          & 6.71e-6          & 0.0195          & 6.65e-4          & 0.293          & 6.69e-5          & 0.0768            \\
4aux    &  1.26e-3         &   0.345          & 4.24e-4          & 0.110           & 6.13e-4          & \textbf{0.238} & 2.56e-4          & 0.0991       \\
5       &  1.42e-3         &   0.456          & 1.16e-5          & 0.0303          & 7.28e-4          & 0.378          & 8.58e-5          & 0.101       \\
5aux    &  1.46e-3         &   \textbf{0.336} & 5.05e-4          & 0.124           & \textbf{6.03e-4} & 0.398          & 2.40e-4          & 0.0818        \\
\hline
\end{tabular}
\caption{Best $L_2$ and $L_\infty$ losses achieved for fully-connected and convolutional networks using both SGD and FAS with stability control. Level $k$aux refers to the first auxiliary network in the $k$-level hierarchy.}
    \label{tab:deep darcy losses}
\end{table}

\subsection{Experiments with Learning Rates}

\begin{figure}
\centering
\begin{tabular}{cc}
\subfloat{\includegraphics[width=.4\textwidth]{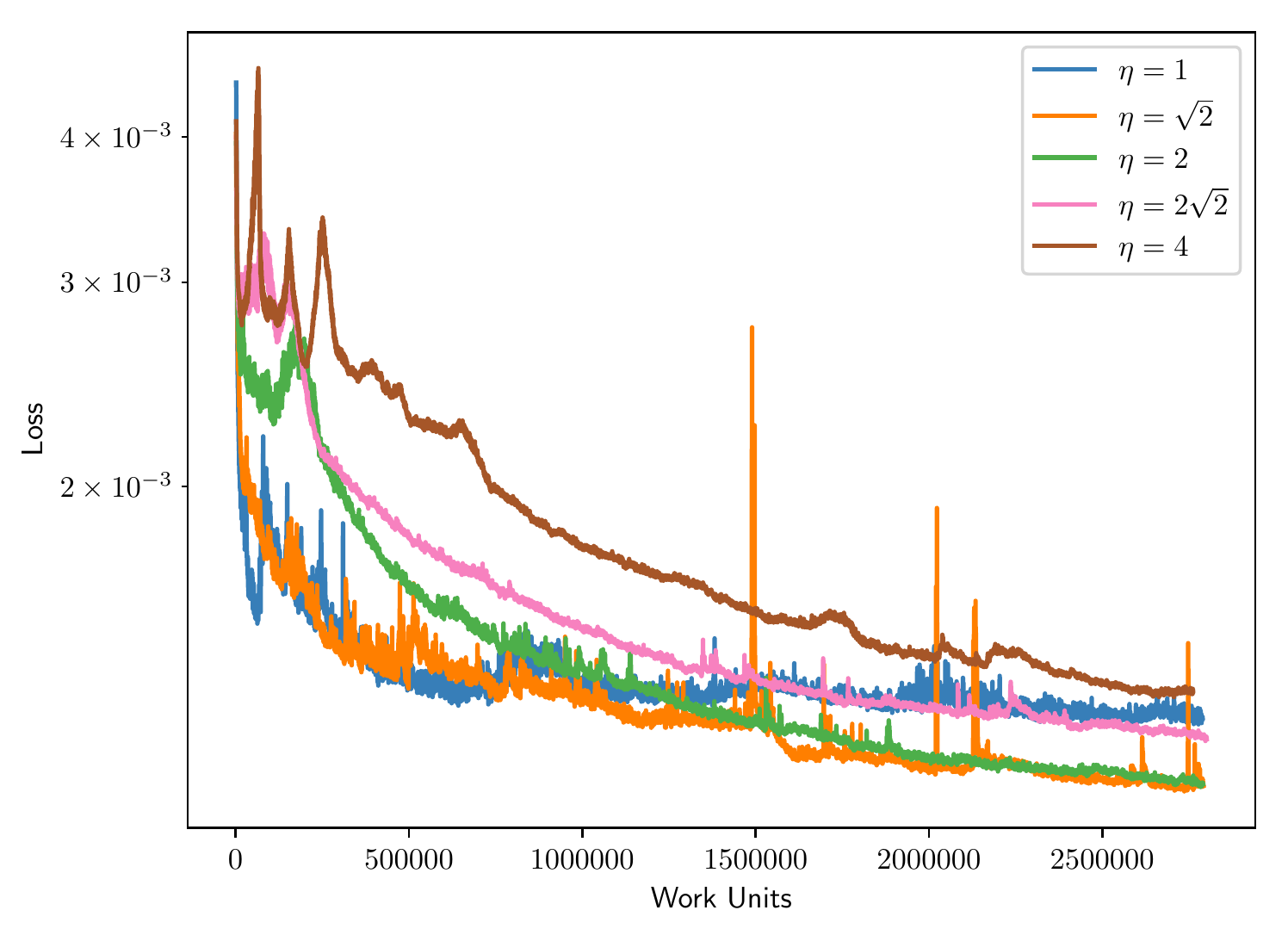}}%
&
\subfloat{\includegraphics[width=.4\textwidth]{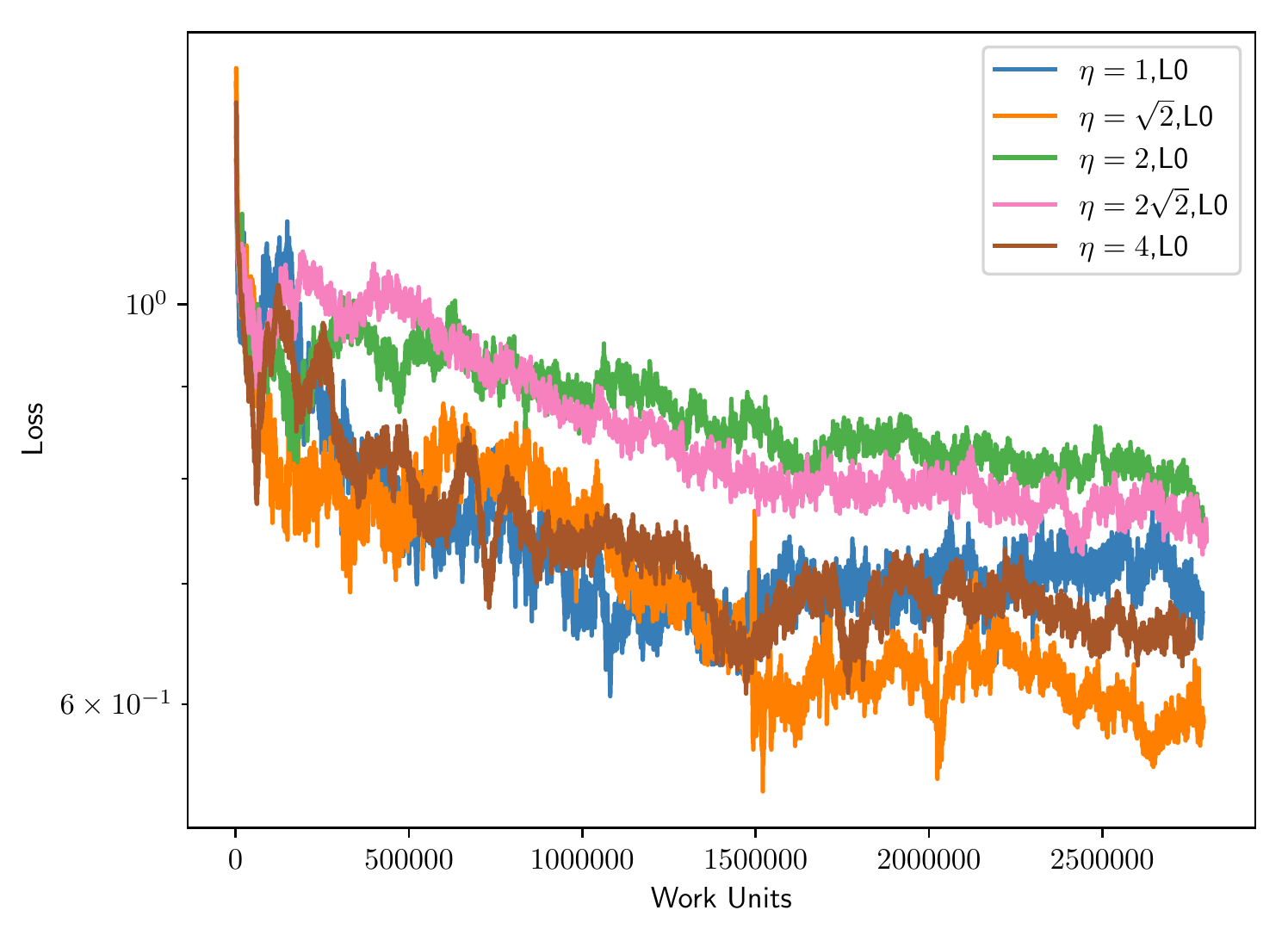}}%
\end{tabular}
\caption{$L_2$ validation loss (left) and $L_\infty$ validation loss (right) for varying hierarchical learning $\eta$ parameters. A smoothing window of size 65 was used }
\label{darcy_learnrate}
\end{figure}

There are many possible hyperparameter and architecture choices one could make, and we cannot present a full exploration of all hyperparameter choices here. We instead present experiments in which we vary how the learning rate changes with hierarchy for the 4-level hierarchy. In the previous section we divided the learning rate by a factor of $2 \sqrt{2}$ as the hierarchy deepens for the first three levels; here we also consider division constants of 1, $\sqrt{2}$, 2, and 4. These results can be see in Figure \ref{darcy_learnrate}. Best losses seen across the training runs, of both validation and training losses, can be see in Table \ref{tab:darcy learnrate losses}.

\begin{table}
\centering
\begin{tabular}{r|l|l|l|l|}
                                                  & \multicolumn{4}{|c|}{Full-Connected Network}                          \\
        & \multicolumn{2}{|c|}{Validation Loss}& \multicolumn{2}{|c|}{Training Loss}  \\
\hline
$\eta$  & $L_2$            & $L_\infty$       &   $L_2$          & $L_\infty$           \\
\hline
1                 &  1.21e-3          &   0.412          & 1.57e-5          & 0.0328             \\
1, Aux1           &  1.40e-3          &   0.347          & 3.00e-4          & 0.0871             \\
$\sqrt 2$         &  \textbf{1.08e-3} &   0.364          & 9.01e-6          & 0.0218    \\
$\sqrt 2$, Aux1   &  1.26e-3          &   \textbf{0.297} & 3.40e-4          & 0.0921                  \\
2                 &  1.09e-3          &   0.602          & 7.10e-6          & 0.0196                 \\
2, Aux1           &  1.24e-3          &   0.345          & 3.83e-4          & 0.106                 \\
$2 \sqrt 2$       &  1.20e-3          &   0.484          & \textbf{6.71e-6} & \textbf{0.0195}                 \\
$2 \sqrt 2$, Aux1 &  1.26e-3          &   0.345          & 4.24e-4          & 0.110             \\
4                 &  1.29e-3          &   0.396          & 1.21e-5          & 0.0260               \\
4, Aux1           &  1.31e-3          &   0.345          & 4.45e-4          & 0.116             \\
\hline
\end{tabular}
\caption{Best $L_2$ and $L_\infty$ losses achieved for fully-connected networks trained using 4-level FAS with varying learning rate decay $\eta$ parameters. Learning rate decays for the first 3 levels of the hierarchy, and then remains the same for the fourth.}
    \label{tab:darcy learnrate losses}
\end{table}

Like previous experiments, we see from these results that the first auxiliary network tends to produce stronger $L_\infty$ performance, which we hypothesize to be due to the reduced degrees of freedom having less opportunity for overfitting. Of these options, $\eta = \sqrt{2}$ seems to perform best: its $L_2$ loss initially drops as quickly as $\eta=1$ but eventually outperforms it, and its auxiliary network achieves the best $L_\infty$ loss.

\section{Conclusion}\label{conclusion}

In this paper we have presented a multilevel-in-width method for training neural networks by developing a variation of the Full Approximation Scheme as originally used in Algebraic Multigrid. We have shown that DNNs trained in this way for regression problems can achieve better generalization power by improving on multiple validation performance metrics. We also observe that, in some cases, the first auxiliary network achieves better worst-case performance than the original network, likely due to having less opportunity for overfitting.

This paper is perhaps most interesting in the context of \cite{planta2021resnets} and \cite{cyr2019multilevelinitialization}, which also showed regularization benefits to multilevel training even though all have taken quite different approaches. Future investigations should include work that seeks to more-fully understand the settings in which multilevel training confers regularization benefits, and to develop principled approaches for making specific algorithmic choices in those settings.

Our open-source PyTorch-based software, MTNN, to implement these methods, is written in a modular fashion to facilitate the study of different algorithmic choices and is available on Github at \url{https://github.com/LLNL/MTNN}. We have also provided open-source data and drivers that allow users to easily recreate our experiments. Readers are encouraged to download and use our software.

\section*{Acknowledgments}
We would like to thank Tom Benson and Ulrike Yang for their insight and suggestions throughout the research and writing of this paper.

\insertbibstyle
\bibliography{mybib}

\end{document}